\documentclass[english]{article}

\usepackage{geometry}
\usepackage[T1]{fontenc}
\usepackage[latin9]{inputenc}
\usepackage{bm}
\usepackage{amsmath}
\usepackage{amssymb}
\usepackage[unicode=true,
 bookmarks=false,
 breaklinks=false,pdfborder={0 0 1},colorlinks=false]
 {hyperref}
\hypersetup{
 colorlinks,citecolor=blue,filecolor=blue,linkcolor=blue,urlcolor=blue}

\makeatletter
\usepackage{amsthm}
\usepackage{cite}  
\usepackage{comment}
\usepackage{natbib}
\usepackage{booktabs,mathtools}

\newcommand{\alglinelabel}{%
  \addtocounter{ALC@line}{-1}
  \refstepcounter{ALC@line}
  \label
}

\usepackage{graphicx}
\usepackage{algorithm}
\usepackage{algorithmic}

\usepackage{float}
\usepackage{multirow}
\usepackage{makecell}
\usepackage{hhline}
\usepackage{caption}
\usepackage{dsfont}

\usepackage{color}
\definecolor{yxc}{RGB}{255,0,0}
\definecolor{yjc}{RGB}{125,0,0}
\definecolor{ytw}{RGB}{255,0,127}
\definecolor{gen}{RGB}{0,0,200}

\allowdisplaybreaks

\usepackage{todonotes}

\newcommand{\blue}[1]{\textcolor{blue}{#1}}

\newcommand{\byes}{{\blue{{\bf YES}}}}
\newcommand{\bno}{{\blue{{\bf NONE}}}}

\makeatletter
\newcommand{\neutralize}[1]{\expandafter\let\csname c@#1\endcsname\count@}
\makeatother

\theoremstyle{plain}
\newtheorem{theorem}{Theorem}[section]
\newtheorem{prop}[theorem]{Proposition}
\newtheorem{lem}[theorem]{Lemma}
\newtheorem{coro}[theorem]{Corollary}
\newtheorem{defi}[theorem]{Definition}
\newtheorem{ass}[theorem]{Assumption}
\newtheorem{manualassinner}{Assumption}
\newenvironment{manualass}[1]{%
  \renewcommand\themanualassinner{#1}%
  \manualassinner
}{\endmanualassinner}
\theoremstyle{remark}

\title{Escaping Saddle Points in Heterogeneous Federated Learning \\ via Distributed SGD with Communication Compression }

\author{Sijin Chen\thanks{Princeton University; email: \texttt{chensj@princeton.edu}.} \\
Princeton
\and 
Zhize Li\thanks{Carnegie Mellon University; emails: \texttt{\{zhizeli, yuejiechi\}@cmu.edu}.} \\
CMU 
\and 
Yuejie Chi\footnotemark[2] \\
CMU 
}
	
\date{\today}

\begin{document}

\newcommand{\RR}{\mathbb{R}}
\newcommand{\CC}{\mathbb{C}}
\newcommand{\TT}{\mathbb{T}}
\newcommand{\eRR}{\overline{\mathbb{R}}}
\newcommand{\HH}{\mathbb{H}}
\newcommand{\EE}{\mathbb{E}}
\newcommand{\dist}{\mbox{dist}}
\newcommand{\nm}[1]{\left\|{#1}\right\|}
\newcommand{\seq}[1]{\{{#1}^k\}_{k\ge 0}}
\newcommand{\sseq}[1]{\{{#1}^{k_j}\}_{j\ge 0}}
\newcommand{\seqq}[1]{\{{#1}_k\}_{k\ge 0}}
\newcommand{\sseqq}[1]{\{{#1}_{k_j}\}_{j\ge 0}}
\newcommand{\ra}{\rightarrow}
\newcommand{\tg}{\tilde{g}}
\newcommand{\tbx}{\tilde{\bm{x}}}

\newcommand{\cf}{{\it cf.}}
\newcommand{\eg}{{\it e.g.}}
\newcommand{\ie}{{\it i.e.}}
\newcommand{\etc}{{\it etc.}}
\newcommand{\res}{{\it res.}}
\newcommand{\ea}{{\it et al.\ }}
\newcommand{\viz}{{\it viz.}}

\newcommand{\BlkDiag}{\mbox{BlkDiag}}

\newcommand*{\vertbar}{\rule[-1ex]{0.5pt}{2.5ex}}
\newcommand*{\horzbar}{\,\rule[.5ex]{2.5ex}{0.5pt}\,}



\newcommand{\MATLAB}{\textsc{Matlab}\xspace}
\newcommand{\sync}{\textsf{sync}\xspace}
\newcommand{\R}{\mathbb{R}}
\newcommand{\F}{\mathcal{F}}
\newcommand{\SO}{\mathcal{SO}}
\newcommand{\OO}{\mathcal{O}}
\newcommand{\tr}{\mathrm{tr}}
\newcommand{\ip}[1]{\left<{#1}\right>}
\newcommand{\sumjii}{\sum_{j\in\mathcal{I}^e_i\cap \mathcal{I}_i}}
\newcommand{\hU}{\hat{\bm{U}}}
\newcommand{\RK}{\left(\mathbf{1}_{1\times K}\otimes\bm R\right)}
\newcommand{\setbar}{\;|\;}
\newcommand{\lcm}{\mathrm{lcm}}
\newcommand{\hbx}{\hat{\bm x}}
\newcommand{\hby}{\hat{\bm y}}
\newcommand{\eqbr}{\phantom{{}=1}}
\newcommand{\avg}[1]{\frac{1}{#1}\sum_{i=1}^{#1}}

\newcommand{\ALG}{Power-EF}

\maketitle

\begin{abstract}
We consider the problem of finding second-order stationary points of heterogeneous federated learning (FL). Previous works in FL mostly focus on first-order convergence guarantees, which do not rule out the scenario of unstable saddle points. Meanwhile, it is a key bottleneck of FL to achieve communication efficiency without compensating the learning accuracy, especially when local data are highly heterogeneous across different clients. Given this, we propose a novel algorithm \ALG{} that only communicates compressed information via a novel error-feedback scheme. To our knowledge, \ALG{} is the first distributed and compressed SGD algorithm that provably escapes saddle points in heterogeneous FL without any data homogeneity assumptions. In particular, \ALG{} improves to second-order stationary points after visiting first-order (possibly saddle) points, using additional gradient queries and communication rounds only of almost the same order required by first-order convergence, and the convergence rate exhibits a linear speedup in terms of the number of workers. Our theory improves/recovers previous results, while extending to much more tolerant settings on the local data. Numerical experiments are provided to complement the theory.
\end{abstract}

\medskip



\section{Introduction}
\label{sec:intro}
 
The prevalence of large-scale data and enormous model size in modern machine learning problems give rise to an increasing interest in distributed machine learning, where a number of clients cooperate to handle the extremely heavy computation in the learning task without the need to move data around. 

We consider a distributed server-client setting. Suppose that each client $i\in[n]$ has access to a local dataset $\mathcal W^{(i)}$ distributed over an unknown space $\Omega$, and a central server maintains a model parameterized by $\bm x\in\R^d$. Given a cost function $F:\R^d\times\Omega\to\R$ that evaluates the performance of a model $\bm x$ on an input data sample $\omega\in\Omega$, the $i$-th local objective function $f_i$ is defined by  
$f_i(\bm x):=\EE_{\omega^{(i)}\sim\mathcal W^{(i)}}[F(\bm x,\omega^{(i)})].$
We would like to find a model parameter $\bm x$ that minimizes the local objectives in an averaged manner, which leads to a finite-sum minimization problem:
\begin{align}\label{eq:dist_ob}
    \min_{\bm x\in\R^d} f(\bm x) :=\frac{1}{n}\sum_{i=1}^nf_i(\bm x),
\end{align}
where the local objective functions $\{f_i\}_{i=1}^n$ and the global objective function $f=\avg{n}f_i$ are in general nonconvex, especially in machine learning applications.

\paragraph{Heterogeneous federated learning.} Assumptions on data homogeneity across the clients can be deployed to underplay this problem to a certain extent, since intuitively, there are less disagreements across the local objectives to reconcile. For example, each local dataset $\mathcal W^{(i)}$ may take similar distributions, or may be uploaded to a data center that maintains global knowledge \citep{konevcny2016federatedlearning}. However, in many real applications such as Internet of Things (IoT) \citep{savazzi2020federated,nguyen2021federated}, smart healthcare \citep{xu2021federated}, and networked model devices \citep{kang2020reliable}, such assumptions become impractical in that local datasets display a strongly heterogeneous pattern, while they should not be exchanged or exposed to a third party due to privacy sensitivity or communication infeasibility \citep{konevcny2016federatedlearning}. These thorny scenarios of data heterogeneity correspond to a framework for distributed learning, namely federated learning (FL) \citep{kairouz2019advances}, which is now accumulating special attention from both academia and industry. The heterogeneous data constitute a major challenge in the distributed optimization problem under federated settings, which we refer to as heterogeneous FL.

\paragraph{Distributed SGD with communication compression.} A prevalent approach to solve \eqref{eq:dist_ob} is  by distributed stochastic gradient descent (SGD) \citep{koloskova2020unified}, a family of algorithms following the essential idea that each client computes its local stochastic gradient and then sends the gradient (or a carefully designed surrogate for the gradient) to the central server for parameter update. Distributed SGD has to take good care of communication efficiency: due to the large client number $n$ \citep{savazzi2020federated} and model scale $d$ \citep{brown2020language} in modern machine learning tasks, the communication cost from the clients to the server becomes the main bottleneck of optimization. Moreover, many resource constraints in real communication systems, such as limited bandwidth and stringent delay requirements, also highlight the importance of establishing efficient communication for the distributed training procedure. 

A natural method to attain communication efficiency is (lossy) compression: one can deploy a \textit{compressor} $\mathcal C:\R^d\to\R^d$ in distributed SGD, which compresses any message $\bm x\in\R^d$ the client would like to send to the server, so that the traffic $\mathcal C(\bm x)$ takes up a smaller bandwidth. In literature, a randomized operator $\mathcal C:\R^d\to\R^d$ is said to be a $\mu$-compressor if the (expected) relative distortion of the compressed output is bounded by $\mu$ \citep{stich2018sparsified,richtarik2021ef21,fatkhullin2021ef21,Huang2022lower}, which helps quantify the information loss due to compression.
 
\paragraph{Motivation.} It has been a recent interest to establish convergence results for distributed SGD with communication compression. Many among these works \citep{stich2018sparsified,koloskova2019decentralizeddeep,xie2020cser,avdiukhin2021escaping,Huang2022lower} assume bounded local gradients $\nm{\nabla f_i(\bm x)}^2\leq G^2$, or bounded dissimilarity of local gradients $\frac{1}{n}\nm{\sum_{i=1}^n\nabla f_i(\bm x)-\nabla f(\bm x)}^2\leq G^2$, reflecting a reliance on data homogeneity that fails to hold in heterogeneous FL. Another body of the works \citep{richtarik2021ef21,fatkhullin2021ef21,richtarik20223pc,zhao2022beer}, although allowing heterogeneous data, only ensures first-order optimality, i.e. convergence to an  $\epsilon$-optimal first-order stationary point $\bm x$ with $\nm{\nabla f(\bm x)}\leq\epsilon$, which does not suffice to justify the goodness of the solution in the nonconvex setting where saddle points are abundant and do not necessarily lead to generalizable performance \citep{dauphin2014identifying}. It is then important to obtain second-order convergence guarantees that ensure the algorithm escapes the saddle points and converges to an $\epsilon$-optimal second-order stationary point, with an additional control on the Hessian positive-definiteness that says $-\lambda_{\min}(\nabla^2f(\bm x))\leq O\left(\sqrt{\epsilon}\right)$. Despite the growing literature of saddle-point escaping algorithms in the centralized setting \citep{jin2019nonconvex,ge2015escaping,li2019ssrgd,daneshmand2018escaping}, to the best of our knowledge, no existing distributed SGD algorithms succeed with second-order guarantees in the presence of both communication compression and data heterogeneity. In summary, the current research sparked a natural question as the primary concern of this paper: 
\begin{center}
 \textit{On heterogeneous data, is there a distributed SGD algorithm with communication compression that attains second-order convergence guarantees for nonconvex problems?}
\end{center}

\subsection{Our contribution}

To the best of our knowledge, this work is the first to answer the above question affirmatively. Our specific contributions are as follows.
\begin{itemize}
    \item \textbf{A novel error-feedback mechanism:} we propose \ALG{}, a new distributed SGD algorithm that contains a novel error-feedback mechanism for communication compression.
    \item \textbf{First-order convergence:} we prove that, with high probability, \ALG{} converges to $\epsilon$-optimal first-order stationary points within $\tilde{O}\left(\frac{1}{n\epsilon^4}+\frac{1}{\mu^{1.5}\epsilon^3}\right)$ stochastic gradient queries and communication rounds. The algorithm shows a linear speedup pattern in that the convergence rate benefits with the increase of the number of workers $n$.
    \item \textbf{Second-order convergence:} we prove that, with high probability, \ALG{} escapes the saddle points and converges to $\epsilon$-optimal second-order stationary points within $\tilde{O}\left(\frac{1}{n\epsilon^4}+\frac{1}{\mu^{1.5}\epsilon^3}+\frac{\mu n+1}{\mu^3\epsilon^{2.5}}\right)$ stochastic gradient queries and communication rounds. This suggests that \ALG{} finds second-order stationary points with almost the same order of gradient and communication complexities as it takes to for first-order convergence.
    \item \textbf{Convergence under arbitrary data heterogeneity:} importantly, the theory of \ALG{} does not require assumptions on data similarity between different clients, thus allowing arbitrary heterogeneity in federated learning tasks.
\end{itemize}
See also Table \ref{tab:results-gradient-complexity} and \ref{tab:results-communication-rounds} for a detailed comparison between our proposed method and existing algorithms.

\begin{table*}[!ht]
    \caption{Comparison of algorithms using \emph{stochastic gradients} for nonconvex problems. {Stochastic gradient complexity} refers to the number of stochastic gradient queries required to converge to $\epsilon$-optimal first-order or $\epsilon$-optimal second-order stationary points, and $\mu$ refers to the parameter of the compressor.}
    \label{tab:results-gradient-complexity}
    \renewcommand{\arraystretch}{1.5}
    \centering
    \resizebox{\textwidth}{!}{
        \begin{tabular}{|c|c|c|c|c|c|c|}
        \hline
        \bf Algorithm & \makecell{\bf Stochastic gradient\\\bf complexity} & \makecell{\bf Result \\ \bf guarantee} & \makecell{\bf Data homogeneity\\\bf assumption} & \bf Distributed? & \bf Compression? \\
        \hline


        \makecell{SGD \\ \citep{ghadimi2016mini}} & ${O}\left(\frac{1}{\epsilon^4}\right)$  & 1st-order & not applicable & NO & NO\\ \hline

        \makecell{Compressed SGD \\ \citep{avdiukhin2021escaping}} & ${O}\left(\frac{1}{\epsilon^4} + \frac{1}{\mu\epsilon^3}\right)$  & 1st-order & not applicable & NO & \byes \\ \hline
        
        \makecell{CHOCO-SGD \\ \citep{koloskova2019decentralizeddeep}} & ${O}\left(\frac{1}{n\epsilon^4} + \frac{1}{\mu\epsilon^3}\right)$  & 1st-order & bounded gradient & \byes & \byes \\ \hline

        \makecell{CSER \\ \citep{xie2020cser}} & ${O}\left(\frac{1}{n\epsilon^4} + \frac{1}{\mu\epsilon^3}\right)$  & 1st-order & bounded gradient & \byes & \byes \\ \hline




        \makecell{NEOLITHIC \\ \citep{Huang2022lower}} & $\tilde{O}\left(\frac{1}{n\epsilon^4}+\frac{1}{\mu\epsilon^2}\right)$  & 1st-order & gradient similarity & \byes & \byes \\ \hline

        \makecell{EF21-SGD \\ \citep{fatkhullin2021ef21}} & ${O}\left(\frac{1}{\mu^3\epsilon^4} + \frac{1}{\mu\epsilon^2}\right)$  & 1st-order & \bno & \byes & \byes \\ \hline

        \makecell{\bf \ALG{}\\(Algorithm \ref{alg:ALG})} & $\tilde{O}\left(\frac{1}{n\epsilon^4} + \frac{1}{\mu^{1.5}\epsilon^3}\right)$ & 1st-order & \bno & \byes & \byes \\  \hhline{======}


        \makecell{Noisy SGD \\ \citep{ge2015escaping}} & $ \mathrm{poly}(\frac{1}{\epsilon})$  & 2nd-order & not applicable & NO & NO \\ \hline
        
        \makecell{CNC-SGD \\ \citep{daneshmand2018escaping}} & $\tilde{O}\left(\frac{1}{\epsilon^5}\right)$  & 2nd-order & not applicable & NO & NO \\ \hline

        \makecell{Perturbed SGD \\ \citep{jin2019nonconvex}} & $\tilde{O}\left(\frac{1}{\epsilon^4}\right)$  & 2nd-order & not applicable & NO & NO \\ \hline

        \makecell{Compressed SGD \\ \citep{avdiukhin2021escaping}} & $\tilde{O}\left(\frac{1}{\epsilon^4}+\frac{1}{\mu\epsilon^3}+\frac{1}{\mu^2\epsilon^{2.5}}\right)$  & 2nd-order & not applicable & NO & \byes \\ \hline

        \makecell{\bf \ALG{}\\(Algorithm \ref{alg:ALG})} & $\tilde{O}\left(\frac{1}{n\epsilon^4}+\frac{1}{\mu^{1.5}\epsilon^3}+\frac{\mu n+1}{\mu^3\epsilon^{2.5}}\right)$ & 2nd-order & \bno & \byes & \byes \\  \hline
        \end{tabular}
    }
    \vspace{3mm}
\end{table*}

\begin{table*}[!h]
    \caption{Comparison of \emph{distributed and compressed} algorithms using stochastic gradients for nonconvex problems. {Communication rounds} refers to the number of compressed messages transmitted between clients and the server.}
    \label{tab:results-communication-rounds}
    \renewcommand{\arraystretch}{1.5}
    \centering
    \resizebox{\textwidth}{!}{
    \begin{tabular}{|c|c|c|c|c|}
    \hline
    \bf Algorithm & \makecell{\bf Communication rounds} & \makecell{\bf Result guarantee} & \makecell{\bf Data homogeneity assumption} \\
    \hline


    \makecell{CHOCO-SGD \\ \citep{koloskova2019decentralizeddeep}} & ${O}\left(\frac{1}{n\epsilon^4}+\frac{1}{\mu\epsilon^3}\right)$  & 1st-order & bounded gradient \\ \hline

    \makecell{CSER \\ \citep{xie2020cser}} & ${O}\left(\frac{1}{n\epsilon^4}+\frac{1}{\mu\epsilon^3}\right)$  & 1st-order & bounded gradient \\ \hline

    \makecell{NEOLITHIC \\ \citep{Huang2022lower}} & $\tilde{O}\left(\frac{1}{n\epsilon^4}+\frac{1}{\mu\epsilon^2}\right)$  & 1st-order & gradient similarity \\ \hline

    \makecell{EF21-SGD \\ \citep{fatkhullin2021ef21}} & $O\left(\frac{1}{\mu^3\epsilon^4}+\frac{1}{\mu\epsilon^2}\right)$ & 1st-order & \bno \\ \hline


    \makecell{\bf \ALG{}\\(Algorithm \ref{alg:ALG})} & $\tilde{O}\left(\frac{1}{n\epsilon^4}+\frac{1}{\mu^{1.5}\epsilon^3}\right)$ & 1st-order & \bno \\  \hhline{====}


    \makecell{\bf \ALG{}\\(Algorithm \ref{alg:ALG})} & $\tilde{O}\left(\frac{1}{n\epsilon^4}+\frac{1}{\mu^{1.5}\epsilon^3}+\frac{\mu n+1}{\mu^3\epsilon^{2.5}}\right)$ & 2nd-order & \bno \\  \hline
    \end{tabular}
    }
    
\end{table*}

\subsection{Related works}

\paragraph{Communication compression.}
A communication operator, or a compressor, is deployed to reduce the communication cost in distributed SGD. Various instances of compressors include Quantized SGD \citep{alistarh2017qsgd} that rounds real-valued gradient vectors to discrete buckets, Sign SGD \citep{bernstein2018signsgd} that represents the gradient with the sign of each coordinate, Top-$k$ \citep{stich2018sparsified} that selects $k$ coordinates out of the total dimension $d$ with the largest magnitudes, and Random-$k$ \citep{stich2018sparsified} that performs the above selection uniformly at random, among others. Regardless of the specific design, a general biased compressor is characterized by a parameter $\mu\in(0,1]$ that controls the aforementioned distortion of the operator.

With a compressor at hand, one also needs a mechanism that specifies what message should be compressed and transmitted between clients. A naive, prototypical mechanism is to directly replace the gradient with its compressed version and then conduct the update step as what is done in regular SGD, for example $\bm x_{t+1}=\bm x_t-\eta\cdot\mathcal C(\tilde{\nabla}f(\bm x_t))$ or its momentum variants. This mechanism underpins \citet{alistarh2017qsgd,bernstein2018signsgd}, among others. However, error may accumulate in this simple replacement due to the lossy compression and menace its convergence. Various works propose new mechanisms to properly handle the error to boost the convergence performance, including Error-Feedback \citep{seide20141,stich2018sparsified,karimireddy2019error,avdiukhin2021escaping,li2022soteriafl} 
and its variants \citep{richtarik2021ef21,fatkhullin2021ef21,Huang2022lower}, with adaptations to decentralized optimization \citep{koloskova2019decentralizeddeep,koloskova2019decentralized,zhao2022beer}.   Most of the works guarantee first-order convergence subject to different levels of assumptions on data homogeneity, cf. Tables \ref{tab:results-gradient-complexity} and \ref{tab:results-communication-rounds}.

\paragraph{Second-order convergence of gradient methods.}
It is well-known that gradient methods converge to first-order stationary points \citep{nesterov2014introductory}. In non-convex problems, however, first-order convergence can be easily attacked by saddle points that may trap the GD trajectory. It is therefore important to investigate whether the algorithm is capable of escaping saddle points and converging to second-order stationary points.
Asymptotically, \citet{lee2016gradient} proved that GD with random initialization converges to a local minimum almost surely. However, the algorithm may still have to take an exponential time to escape the saddle points \citep{du2017gradient}.

As to the polynomial-time guarantees, it is known that perturbing the gradient with isotropic noise helps GD converge to local minimizers \citep{ge2015escaping,jin2017escape}. The perturbation technique gives rise to similar guarantees for other gradient methods, from SGD \citep{jin2019nonconvex} to SVRG \citep{ge2019stable} and stochastic recursive gradient descent \citep{li2019ssrgd}. On the other hand, instead of gradient perturbation, \citet{daneshmand2018escaping} establishes the saddle-escaping property of SGD under an additional Correlated Negative Curvature (CNC) assumption regarding the statistical property of the stochastic gradient oracle. 

Recently, \citet{avdiukhin2021escaping} leverages the perturbation technique to analyze the second-order stationarity of SGD with communication compression. The derivation is based on {\em single-node} implementation, which does not directly extend to the distributed settings. Further, it requires a conditional reset procedure in each iteration to achieve second-order convergence, at the expense of high communication cost as the server has to collect and maintain the local error terms using an \textit{uncompressed} channel. Therefore, it remains obscure if the results therein still apply to the distributed setting with communication efficiency demands. 

\subsection{Notation}

Throughout, we use lowercase boldface letters to denote vectors, and uppercase boldface letters to denote matrices. Let $\bm I$ be the identity matrix. Let $\ip{\bm u,\bm v}:=\bm u^\top\bm v$ denote the standard Euclidean inner product of two vectors $\bm u$ and $\bm v$. The operator $\nm{\cdot}$ denotes the Euclidean norm when exerted on a vector, i.e. $\nm{\bm x}:=\sqrt{\ip{\bm x,\bm x}}=\sqrt{\bm x^\top\bm x}$, and denotes the spectral (operator) norm when exerted on a matrix, i.e. $\nm{\bm A}:=\sup_{\bm x}\nm{\bm A\bm x}/\nm{\bm x}$. In addition, we use the standard order notation $O(\cdot)$ to hide absolute constants, and $\tilde{O}(\cdot)$ to hide polylog factors.

\section{Problem Formulation}

This paper is concerned with solving the nonconvex finite-sum minimization problem in a federated setting, while each client should only query a local stochastic gradient oracle, and communicate their information with the server in an efficient manner using compression. We detail this formulation in the following.

\label{sec:formulation}
\subsection{Nonconvex finite-sum minimization}
Recall that we consider a federated optimization problem of finding an optimal parameter $\bm x$ to minimize the local objectives $\{f_i\}_{i=1}^n$ in an averaged manner, which is stated as an unconstrained finite-sum minimization problem:
\begin{align*}
    \min_{\bm x\in\R^d} f(\bm x):= \frac{1}{n}\sum_{i=1}^nf_i(\bm x),
\end{align*}
where $\{f_i\}_{i=1}^n$'s are the local objective functions, and $n$ is the number of clients.

We focus on the case where the objective functions are nonconvex, subject to the following assumptions.
\begin{ass}\label{ass:f-min}
There exists some $f_{\min}>-\infty$ such that $f(\bm x)\geq f_{\min}$ for all $\bm x\in\R^d$.
\end{ass}

We will leverage the boundedness in Assumption \ref{ass:f-min} to establish first-order convergence results. For second-order results, similar to \citet{avdiukhin2021escaping}, the following alternative is required.

\begin{manualass}{\ref*{ass:f-min}*}\label{ass:f-max}
There exists some $f_{\max}<\infty$ such that $|f(\bm x_1)-f(\bm x_2)|\leq f_{\max}$ for all $\bm x_1,\bm x_2\in\R^d$.
\end{manualass}


Besides boundedness, we also assume the smoothness of $f$.

\begin{ass}\label{ass:f-L}
$f$ is differentiable and $L$-smooth, i.e. $$\nm{\nabla f(\bm x_1)-\nabla f(\bm x_2)}\leq L\nm{\bm x_1-\bm x_2}, \quad\forall\bm x_1,\bm x_2\in\R^d.$$
\end{ass}

In the same spirit as what we do for the boundedness assumption, we need to further assume a Lipschitz property of the Hessian to prove second-order results.

\begin{ass}\label{ass:f-rho}
$f$ is twice differentiable and $\rho$-Hessian Lipschitz, i.e.,
$$\nm{\nabla^2f(\bm x_1)-\nabla^2f(\bm x_2)}\leq\rho\nm{\bm x_1-\bm x_2}, \quad\forall\bm x_1,\bm x_2\in\R^d.
$$
\end{ass}
We emphasize that no assumption is made on the boundedness of, or similarity between, the local gradients. 

\subsection{Local stochastic gradient oracle}
Each client $i$ is allowed to query a local stochastic gradient oracle $\tilde{\nabla}f_i$.

\begin{ass}\label{ass:stoch-L}
    Each $\tilde{\nabla} f_i$ is $\tilde{L}_i$-Lipschitz, i.e.
    \begin{align*}
        \nm{\tilde{\nabla}f_i(\bm x_1)-\tilde{\nabla}f_i(\bm x_2)}\leq\tilde{L}_i\nm{\bm x_1-\bm x_2},\quad\forall\bm x_1,\bm x_2\in\R^d.
    \end{align*}
\end{ass}
Based on Assumption \ref{ass:stoch-L}, it is straightforward to verify that the global stochastic gradient $\tilde{\nabla}f$ is $\tilde L$-smooth with $\tilde L:=\sqrt{\avg{n}\tilde L_i^2}$.

Besides smoothness, the stochastic gradients should also approximate the true gradients.
\begin{ass}\label{ass:nSG}
For any $\bm x\in\R^d$, the mutually independent stochastic gradient oracles $\tilde{\nabla} f_i$ satisfy
\begin{align*}
\EE\left[\tilde{\nabla} f_i(\bm x)\right] &=\nabla f_i(\bm x),\\
\Pr\left(\nm{\tilde{\nabla} f_i(\bm x)-\nabla f_i(\bm x)}\geq t\right) & \leq2\exp\left(-\frac{t^2}{2\sigma^2}\right)
\end{align*}
for all $ t\geq0$ and some $\sigma>0$. 
\end{ass}
Assumption \ref{ass:nSG} is a high-probability variant of the commonly-used bounded variance assumption, stated in expectation. Switching to such a high-probability variant is again necessary for second-order analysis \citep{jin2019nonconvex,li2019ssrgd} because we aim at a convergence guarantee with probability bounds. 

Additionally, we introduce the mini-batch version of the stochastic gradient. For integer $k$, let $\tilde{\nabla}f_i(\bm x)^{(1)}$, $\ldots$, $\tilde{\nabla}f_i(\bm x)^{(k)}$ be the $k$ independent queries to the stochastic oracle at $\bm x$. The mini-batch gradient is defined as their average, i.e. $\tilde{\nabla}_k f_i(\bm x)=\frac{1}{k}\sum_{j=1}^k\tilde{\nabla}f_i(\bm x)^{(j)}$. 

\subsection{Communication compression}
To enable efficient communication over bandwidth-limited scenarios, our setting requests that the communication between the clients and the server should be compressed according to a possibly randomized scheme $\mathcal C$. Specifically, for any input $\bm x\in\R^d$, the scheme outputs a surrogate $\mathcal C(\bm x)\in\R^d$ so that the transmission of $\mathcal C(\bm x)$ between machines would take up a smaller bandwidth than the direct transmission of $\bm x$.
\begin{defi}\label{defi:C}
A possibly random mapping $\mathcal C:\R^d\to\R^d$ is said to be a $\mu$-compressor for some $\mu\in(0,1]$ if
$$\nm{\bm x-\mathcal C(\bm x)}^2\leq(1-\mu)\nm{\bm x}^2, \quad\forall\bm x\in\R^d.
$$
\end{defi}
Definition \ref{defi:C} slightly deviates from the conventional definition that controls the expected distortion, i.e., $\EE[\nm{\bm x-\mathcal C(\bm x)}^2]\leq(1-\mu)\nm{\bm x}^2$, to facilitate the derivation of high-probability results. Examples of compressors that satisfy Definition \ref{defi:C} include Top-$k$ \citep{stich2018sparsified} and a family of compressors named general biased rounding \citep{beznosikov2020biased}.

\section{Proposed Algorithm}
\label{sec:algorithm}
This section introduces our proposed algorithm \ALG{} that is suitable to heterogeneous FL with communication compression.

\subsection{Fast Compressed Communication}
We first introduce the Fast Compressed Communication (FCC) module proposed by \citet{Huang2022lower}, which is deployed at each client in their compressed SGD algorithm NEOLITHIC. For input $\bm x\in\R^d$, the FCC module with parameter $p\in\mathbb Z_+$ recursively computes the residual $\{\bm v_i\}_{i=1}^p$ for $p$ rounds, where 
$$\bm v_1=\bm x;\quad\bm v_{i}=\bm x-\sum_{j=1}^{i-1}\mathcal C(\bm v_j),\quad i=2,...,p.$$
It then outputs $\mathrm{FCC}_p(\bm x)=\sum_{i=1}^p\mathcal C(\bm v_i)$. To transmit the output to the server efficiently, the client transmits the set of compressed vectors $\{\mathcal C(\bm v_i)\}_{i=1}^p$ through the channel, and the exact output is assembled by summation on the server side.

Defining $\mathcal D:\bm x\mapsto\bm x-\mathcal C(\bm x)$, one can observe that $\mathrm{FCC}_p(\bm x)=\bm x-\mathcal D^p(\bm x)$. In fact, the FCC module is able to refine the compression loss by harnessing the contraction property of $\mathcal D$. Specifically, $\mathcal D$ is a contraction because $\nm{\mathcal D(\bm x)}^2\leq(1-\mu)\nm{\bm x}^2$ due to Definition \ref{defi:C}. Hence, the error of the FCC module $\nm{\bm x-\mathrm{FCC}_p(\bm x)}^2\leq(1-\mu)^p\nm{\bm x}^2$ enjoys a geometric decay with $p$. 

\subsection{\ALG{}}

\begin{algorithm*}[t]
\caption{\ALG{}}\label{alg:ALG}
\begin{algorithmic}[1]
\STATE \textbf{Input}: $\bm x_0$, step size $\eta$, contraction exponent $p$, perturbation radius $r$, number of iterations $T$

\STATE{\textbf{Initialization}: $\bm e_0^{(i)}\leftarrow\bm 0$, $\bm e_{-1}^{(i)}\leftarrow\bm0$, $\bm g_{-1}^{(i)}\leftarrow\bm0$, $\bm g_{-1}\leftarrow\bm0$}

\FOR {$t=0,1,2,...,T-1$}
    \FOR{parameter server}
    \STATE{sample $\bm\xi_t\sim\mathcal N(\bm 0,\frac{r^2}{npd}\bm I)$} 
    \STATE{broadcast $\bm\xi_t$ to every client}
    \ENDFOR

    \FOR{client $i=1,2,...,n$ in parallel}

        \STATE{$\bm w_{t}^{(i)}\leftarrow\mathrm{FCC}_p(\bm e_t^{(i)}-\bm e_{t-1}^{(i)})=\sum_{\ell=1}^p\mathcal C(\bm v_\ell^{(i)})$}\alglinelabel{line:w-update}

        \STATE{$\bm c_t^{(i)}\leftarrow\mathcal{C}(\bm e_t^{(i)}+\tilde{\nabla}_p f_i(\bm x_{t})+\bm\xi_t-\bm g_{t-1}^{(i)}-\bm w_{t}^{(i)})$}\alglinelabel{line:c-update}

        \STATE{$\bm g_t^{(i)}\leftarrow\bm g_{t-1}^{(i)}+\bm w_{t}^{(i)}+\bm c_t^{(i)}$} \hfill \COMMENT{Feedback the local gradient estimate}\alglinelabel{line:g-update}

        \STATE{$\bm e_{t+1}^{(i)}\leftarrow\bm e_{t}^{(i)}+\tilde{\nabla}_p f_i(\bm x_{t})+\bm\xi_{t}-\bm g_{t}^{(i)}$} \hfill \COMMENT{Update the error}\alglinelabel{line:e-update}

        \STATE{upload $\bm c_t^{(i)}$ and $\{\mathcal C(\bm v_\ell^{(i)})\}_{\ell=1}^p$ to server}
    \ENDFOR
    \FOR{parameter server}

    \STATE{$\bm g_{t}\leftarrow\bm g_{t-1}+\avg{n}\sum_{\ell=1}^p\mathcal C(\bm v_{\ell}^{(i)})+\avg{n}\bm c_t^{(i)}$}\alglinelabel{line:global-g-update} \hfill \COMMENT{Prepare the global gradient}

    \STATE{$\bm x_{t+1}\leftarrow\bm x_{t}-\eta\bm g_{t}$} \alglinelabel{line:x-update}  \hfill \COMMENT{Update the model}

    \STATE{broadcast $\bm x_t$ to every client}
    \ENDFOR

\ENDFOR
\end{algorithmic}
\end{algorithm*}

We integrate the FCC module into our algorithm \ALG{}, as summarized in Algorithm \ref{alg:ALG}. The algorithm takes as input an initial model $\bm x_0$, step size $\eta$, FCC parameter $p$, perturbation radius $r$, and the number of iterations $T$. After a simple initialization procedure, \ALG{} iteratively produces a sequence $\{\bm x_t\}_{t=0}^T$ to gradually update the initial model by SGD-type descent. In each iteration, we use the accumulated gradient $\tilde{\nabla}_pf_i$ to balance the number of communication rounds and stochastic gradient complexity. Each iteration of \ALG{} contains four conceptual stages interpreted as follows.
\begin{itemize}
\item {\bf Feedback the local gradient estimate.}
We intend to use $\bm e_t^{(i)}$, the error up to the \textit{last} iteration, to feedback our estimate of the local gradient $\bm g_t^{(i)}$ for the \textit{current} round. Firstly, based on the error, the client invokes FCC module to compute the feedback term $\bm w_t^{(i)}+\bm c_t^{(i)}$ (Line \ref{line:w-update}--\ref{line:c-update}). Then each client $i$ gets its current gradient estimate $\bm g_t^{(i)}$ by complementing the existing estimate $\bm g_{t-1}^{(i)}$ with the feedback term (Line \ref{line:g-update}).

\item {\bf Update the error.}
Upon completion of the feedback, we increase the error term by the discrepancy between the real stochastic gradient (after artificial perturbation) and our local estimate $\bm g_t^{(i)}$ (Line \ref{line:e-update}). In this way, the error term essentially stores the \textit{cumulative} estimation discrepancy of $\bm g_t^{(i)}$, which is ready for feedback again on the next run.

\item {\bf Prepare the global gradient estimate.}
The update of global gradient estimate is conducted on a par with the local update method in an averaged manner (Line \ref{line:global-g-update}), so that we always have $\bm g_t=\avg{n}\bm g_t^{(i)}$.

\item {\bf Update the model.}
Finally, the server updates the current model $\bm x_t$ by a descending step along our global gradient estimate $\bm g_{t}$ (Line~\ref{line:x-update}).

\end{itemize}

\subsection{Discussion}\label{subsec:discussion}

At its core, \ALG{} benefits from the power contraction underlying the FCC module to upgrade the classical error-feedback mechanism \citep{avdiukhin2021escaping,stich2018sparsified}, hence the name. Specifically, our algorithm inherits the classical design of error term to track the cumulative discrepancy of gradient estimation (Line \ref{line:e-update}), but refines the way errors are used to feedback the current gradient estimation by the FCC module. Moreover, while still guaranteeing second-order results, \ALG{} manages to remove from the prior work \citep{avdiukhin2021escaping} an expensive procedure of conditinal reset that inevitably occupies the uncompressed bandwidth.

\paragraph{Data heterogeneity.} Mathematically, our mechanism is able to induce an error term recurrence irrelavent to local gradients, thus circumventing from data similarity assumptions. This favorable property originates from our design of \ALG{}, which is nontrivially different from the existing NEOLITHIC \citep{Huang2022lower} algorithm where FCC module also plays a part. For example, NEOLITHIC inputs the gradient estimate to FCC while we input the estimation discrepancy, and error terms are also computed distinctly. As a notable result, contrary to our algorithm, the theory of NEOLITHIC still has to assume local gradient similarity.

\paragraph{Gradient perturbation.} We add an isotropic Gaussian noise to each stochastic gradient to help the model escape from saddle points. Intuitively, around saddle points, the isotropic perturbation ensures that the SGD trajectory can traverse a sufficient distance along the descending direction, i.e. the eigenvector of Hessian $\nabla^2f(\bm x_t)$ with a negative eigenvalue, thus escaping the saddle region and gaining an objective decrease. The perturbation is not required for first-order convergence, in which case one can safely set $r=0$.

\section{Performance Guarantees}
\label{sec:theory}
In this section, we state the theoretical guarantees for \ALG{}, where the proofs are deferred to the appendix. To begin, we first define the first-order and second-order approximate stationarity conditions.
\begin{defi}\label{defi:FOSP}
$\bm x\in\R^d$ is said to be an $\epsilon$-optimal first-order stationary point ($\epsilon$-FOSP) if $\nm{\nabla f(\bm x)}\leq\epsilon$.
\end{defi}
\begin{defi}\label{defi:SOSP}
Suppose that $\bm x\in\R^d$ is an $\epsilon$-FOSP. Then, $\bm x$ is said to be an $\epsilon$-optimal second-order stationary point ($\epsilon$-SOSP) if
$$\nabla^2 f(\bm x) \succeq -\sqrt{\rho\epsilon}\cdot\bm I.$$
Otherwise, $\bm x$ is said to be an $\epsilon$-strict saddle point.  
\end{defi}
Moreover, we denote $\chi^2:=\sigma^2\log d+r^2$ the effective variance of stochastic gradient and perturbation, and $\Phi$ the initialization quality where
    $$\Phi=\avg{n}\nm{\tilde{\nabla}_p f_i(\bm x_0)+\bm\xi_0}^2+\tilde L\left[f(\bm x_0)-f_{\min} \right].$$

We are now ready to state the main theorems. 

\paragraph{First-order convergence guarantee.} Theorem \ref{thm:FOSP} establishes that \ALG{} converges with high probability to $\epsilon$-FOSP.
\begin{theorem}[Convergence to $\epsilon$-FOSP]\label{thm:FOSP}
Suppose that Assumptions \ref{ass:f-min}, \ref{ass:f-L}, \ref{ass:nSG} hold, and the parameters $T,\eta,p,r$ satisfy
\begin{align*}
    T&=\kappa_T\cdot\max\left\{\frac{f(\bm x_0)-f_{\min}}{\eta\epsilon^2},\frac{\chi^2\iota}{np\epsilon^2}\right\}, \\
    \eta&=\kappa_\eta\cdot\min\left\{\frac{\mu\epsilon}{L\sqrt{\mu\Phi+\frac{\chi^2\iota}{np}}},\frac{np\epsilon^2}{\chi^2L}\right\},\\
    p&=\kappa_p\cdot\frac{1}{\mu}\log\left(\frac{1}{\mu}\right)
\end{align*}
for some constants $\kappa_T,\kappa_\eta,\kappa_p>0$, and the parameter $\iota$ controlling the tightness of the probability bound. Then, with probability at least $1-7e^{-\iota}$, at least 3/4 of the iterates $\{\bm x_t\}_{t=0}^T$ generated by Algorithm \ref{alg:ALG} are $\epsilon$-FOSPs.
\end{theorem}
In words, first-order convergence is guaranteed with high probability (controlled by $\iota$), under an appropriate choice of the algorithm parameters. Note that the theorem does not specify a choice for the perturbation radius $r$, resonating with Section \ref{subsec:discussion} in that perturbation is not required for first-order convergence. Based on Theorem \ref{thm:FOSP}, it is now immediate to compute the gradient complexity and communication rounds of \ALG{} to attain first-order optimality, given by the corollary below.

\begin{coro}[$\epsilon$-FOSP complexity]\label{coro:FOSP}
Under the same setting of Theorem \ref{thm:FOSP}, Algorithm \ref{alg:ALG} requires $\tilde{O}\Big(\frac{1}{n\epsilon^4}+\frac{1}{\mu^{1.5}\epsilon^3}\Big)$ queries to the stochastic gradient oracle and communication rounds.
\end{coro}

\paragraph{Second-order convergence guarantee.} Moving onto the second-order convergence, we have the following theorem.

\begin{theorem}[Convergence to $\epsilon$-SOSP]\label{thm:SOSP}
Suppose that Assumptions \ref{ass:f-max}, \ref{ass:f-L}, \ref{ass:f-rho}, \ref{ass:nSG} hold, and the parameters $T,\eta,p,r$ satisfy
\begin{align}
    T&=\kappa_T\cdot\max\left\{\frac{\iota^5f_{\max}}{\eta\epsilon^2},\frac{\chi^2\iota}{np\epsilon^2}\right\}, \nonumber\\
    \eta&=\kappa_\eta\cdot\min\left\{\frac{\mu\epsilon}{\iota^5L\sqrt{\mu\Phi+\frac{\chi^2\iota}{np}}},\frac{\iota\sigma^2\sqrt{\rho\epsilon}\log d}{L^2\left(np\Phi+\frac{\chi^2\iota}{\mu^2}\right)},\frac{np\epsilon^2}{\iota^5L\chi^2}\right\}, \nonumber\\
    p&=\kappa_p\cdot\frac{1}{\mu}\log\left(\frac{1}{\mu}\right),\nonumber\\
    r&=\kappa_r\cdot\sigma\sqrt{\iota d\log d}\nonumber
\end{align}
for some constants $\kappa_T,\kappa_\eta,\kappa_p,\kappa_r>0$, and the parameter $\iota$ controlling the tightness of the probability bound. Set $\mathcal{I}=\frac{\iota}{\eta\sqrt{\rho\epsilon}}$. Then, with probability at least $1-8T^2(\mathcal I^2+d\mathcal I+\mathcal I+T)e^{-\iota}$, at least half of the iterates $\{\bm x_t\}_{t=0}^T$ generated by Algorithm \ref{alg:ALG} are $\epsilon$-SOSPs.
\end{theorem}

Unlike Theorem~\ref{thm:FOSP}, perturbing the local stochastic gradient with an appropriate radius plays a vital part in the second-order guarantee by assisting the iterates to escape the saddle points. Again, we can
compute the gradient complexity and communication rounds of \ALG{} to attain second-order optimality, given as follows.

\begin{coro}[$\epsilon$-SOSP complexity]\label{coro:SOSP}
Under the same setting of Theorem \ref{thm:SOSP}, Algorithm \ref{alg:ALG} requires $\tilde{O}\Big(\frac{1}{n\epsilon^4}+\frac{1}{\mu^{1.5}\epsilon^3}+\frac{\mu n+1}{\mu^3\epsilon^{2.5}}\Big)$ queries to the stochastic gradient oracle and communication rounds.
\end{coro}

According to the corollaries, \ALG{} improves to second-order stationary points after visiting first-order (possibly saddle) points, using additional gradient queries and communication rounds only of almost the same order required by first-order convergence when $\epsilon$ is typically small to be the dominant parameter. Contrary to another work allowing heterogeneous data \citep{fatkhullin2021ef21}, our convergence rate exhibits a linear speedup in terms of $n$, implying that our algorithm significantly benefits from the distributed framework.

\section{Numerical Experiments}

In this section, we present the performance of \ALG{} in distributed learning experiments to validate its efficiency empirically. We train a ResNet18 model on CIFAR10 dataset \citep{krizhevsky2009learning} using 4 clients and 1 server and compare the performance of various distributed optimization algorithms in the training task, including standard distributed SGD, SGD with naive compression, standard EF, and \ALG{}. All the training procedures take 100 epochs with a step size of $10^{-2}$ and weight decay of $10^{-4}$. For communication compression, we use Top-$k$ compressor that keeps top 1\% coordinates of the largest magnitudes, and \ALG{} is tested with exponent $p=1,4,8$ respectively. The algorithms are implemented on PyTorch \citep{paszke2019pytorch} 2.0.0 and the experiments are conducted on NVIDIA Tesla P100 GPU. 

We summarize the results in Figure \ref{fig:1}. According to Figure \ref{fig:1}(a) and \ref{fig:1}(b), without the feedback mechanism, the lossy compression significantly hinders the the convergence speed and prediction accuracy. On the other hand, EF and \ALG{} have a comparable performance in boosting the training procedure and improving the accuracy. According to Figure \ref{fig:1}(a), increasing the FCC parameter $p$, the convergence speed is almost not affected while an improved test loss is obtained in the final stage. A comparison between the communication efficiency of different algorithms is drawn in Figure \ref{fig:1}(c). The compressor remarkably scales down the communication cost of the training procedure, from nearly $10^4$ GB to no more than $10^2$ GB for 100 epochs.

\begin{figure*}
    \begin{minipage}[b]{0.33\linewidth}
        \centering
        \centerline{\includegraphics[width=1.01\linewidth]{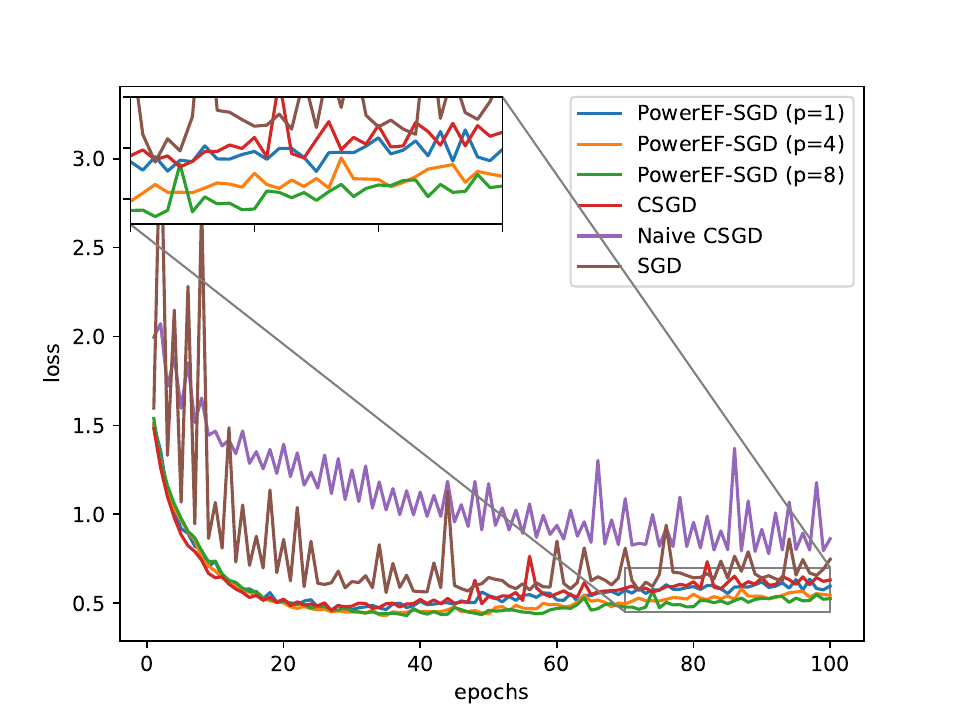}}
        \centerline{(a) test loss}\medskip
    \end{minipage}
    \begin{minipage}[b]{0.33\linewidth}
        \centering
        \centerline{\includegraphics[width=\linewidth]{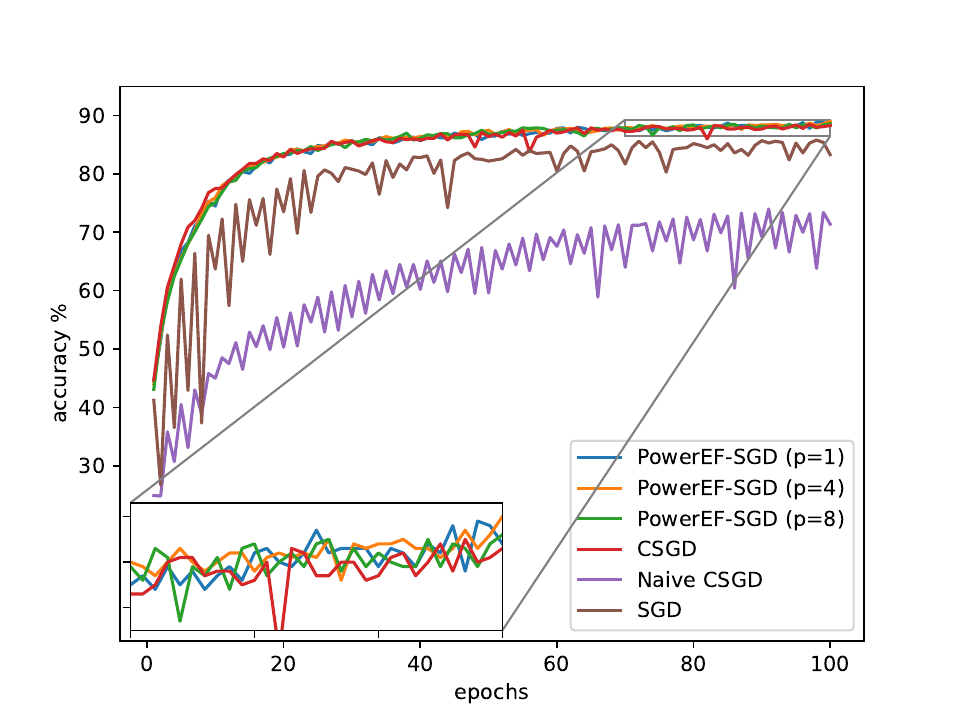}}
        \centerline{(b) test accuracy}\medskip
    \end{minipage}
    \begin{minipage}[b]{0.33\linewidth}
        \centering
        \centerline{\includegraphics[width=\linewidth]{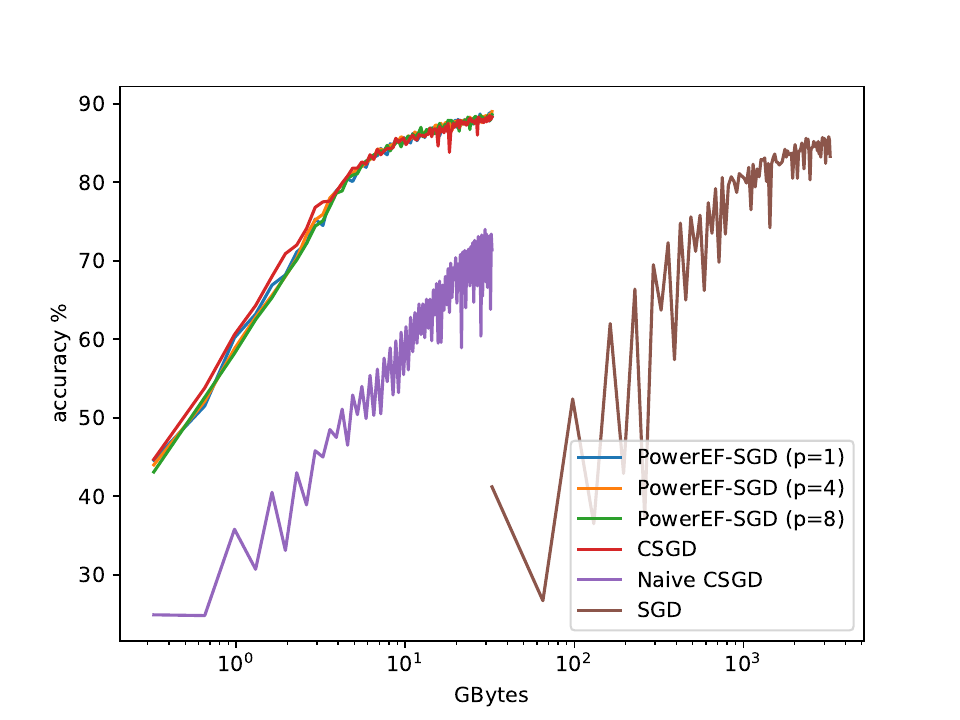}}
        \centerline{(c) test accuracy}\medskip
    \end{minipage}
    \vspace{.01in}
    \caption{Comparison of different algorithms on CIFAR-10 training task. Naive CSGD refers to the naive compression scheme for SGD without any feedback mechanism, and CSGD refers to \citet{avdiukhin2021escaping} that adopts the standard error-feedback. (a) Test loss with respect to training epochs. (b) Test accuracy with respect to training epochs. (c) Test accuracy with respect to size of communication in GBytes.}
    \label{fig:1}
\end{figure*}

\section{Conclusion}
\label{sec:conclusions}
 
In this paper, we propose and analyze \ALG{}, which is the first distributed SGD algorithm with communication compression that provably attains second-order optimality under heterogeneous data, to the best of our knowledge. Specifically, subject to mild and standard assumptions, we show that \ALG{} converges to $\epsilon$-SOSPs with high probability, which is almost on par with the gradient and communication complexity it takes to find $\epsilon$-FOSPs, and the convergence rate shows a linear speedup with respect to $n$. Our theory are complemented by the performance of \ALG{} in the distributed learning experiments. For future work, it will be of great interest to develop privacy-preserving distributed SGD algorithms that can escape saddle points with communication compression.


%

\section*{Acknowledgement}
This work is supported in part by the grants ONR N00014-19-1-2404, NSF CIF-2007911, ECCS-2318441, and AFRL FA8750-20-2-0504. 

\bibliography{FL}

\begin{thebibliography}{}

\bibitem[Alistarh et~al., 2017]{alistarh2017qsgd}
Alistarh, D., Grubic, D., Li, J., Tomioka, R., and Vojnovic, M. (2017).
\newblock {QSGD}: {C}ommunication-efficient {SGD} via gradient quantization and
  encoding.
\newblock In {\em Advances in Neural Information Processing Systems}, pages
  1709--1720.

\bibitem[Avdiukhin and Yaroslavtsev, 2021]{avdiukhin2021escaping}
Avdiukhin, D. and Yaroslavtsev, G. (2021).
\newblock Escaping saddle points with compressed sgd.
\newblock In {\em Advances in Neural Information Processing Systems},
  volume~34, pages 10273--10284. Curran Associates, Inc.

\bibitem[Bernstein et~al., 2018]{bernstein2018signsgd}
Bernstein, J., Wang, Y.-X., Azizzadenesheli, K., and Anandkumar, A. (2018).
\newblock signsgd: Compressed optimisation for non-convex problems.
\newblock In {\em International Conference on Machine Learning}, pages
  560--569. PMLR.

\bibitem[Beznosikov et~al., 2020]{beznosikov2020biased}
Beznosikov, A., Horv{\'a}th, S., Richt{\'a}rik, P., and Safaryan, M. (2020).
\newblock On biased compression for distributed learning.
\newblock {\em arXiv preprint arXiv:2002.12410}.

\bibitem[Brown et~al., 2020]{brown2020language}
Brown, T., Mann, B., Ryder, N., Subbiah, M., Kaplan, J.~D., Dhariwal, P.,
  Neelakantan, A., Shyam, P., Sastry, G., Askell, A., Agarwal, S.,
  Herbert-Voss, A., Krueger, G., Henighan, T., Child, R., Ramesh, A., Ziegler,
  D., Wu, J., Winter, C., Hesse, C., Chen, M., Sigler, E., Litwin, M., Gray,
  S., Chess, B., Clark, J., Berner, C., McCandlish, S., Radford, A., Sutskever,
  I., and Amodei, D. (2020).
\newblock Language models are few-shot learners.
\newblock In {\em Advances in Neural Information Processing Systems},
  volume~33, pages 1877--1901. Curran Associates, Inc.

\bibitem[Daneshmand et~al., 2018]{daneshmand2018escaping}
Daneshmand, H., Kohler, J., Lucchi, A., and Hofmann, T. (2018).
\newblock Escaping saddles with stochastic gradients.
\newblock In {\em Proceedings of the 35th International Conference on Machine
  Learning}, volume~80 of {\em Proceedings of Machine Learning Research}, pages
  1155--1164. PMLR.

\bibitem[Dauphin et~al., 2014]{dauphin2014identifying}
Dauphin, Y.~N., Pascanu, R., Gulcehre, C., Cho, K., Ganguli, S., and Bengio, Y.
  (2014).
\newblock Identifying and attacking the saddle point problem in
  high-dimensional non-convex optimization.
\newblock {\em Advances in neural information processing systems}, 27.

\bibitem[Du et~al., 2017]{du2017gradient}
Du, S.~S., Jin, C., Lee, J.~D., Jordan, M.~I., Singh, A., and Poczos, B.
  (2017).
\newblock Gradient descent can take exponential time to escape saddle points.
\newblock In {\em Advances in Neural Information Processing Systems},
  volume~30. Curran Associates, Inc.

\bibitem[Fatkhullin et~al., 2021]{fatkhullin2021ef21}
Fatkhullin, I., Sokolov, I., Gorbunov, E., Li, Z., and Richt{\'a}rik, P.
  (2021).
\newblock {EF21} with bells \& whistles: Practical algorithmic extensions of
  modern error feedback.
\newblock {\em arXiv preprint arXiv:2110.03294}.

\bibitem[Ge et~al., 2015]{ge2015escaping}
Ge, R., Huang, F., Jin, C., and Yuan, Y. (2015).
\newblock Escaping from saddle points --- online stochastic gradient for tensor
  decomposition.
\newblock In {\em Conference on Learning Theory}, pages 797--842.

\bibitem[Ge et~al., 2019]{ge2019stable}
Ge, R., Li, Z., Wang, W., and Wang, X. (2019).
\newblock Stabilized {SVRG}: Simple variance reduction for nonconvex
  optimization.
\newblock In {\em Conference on Learning Theory}, pages 1394--1448.

\bibitem[Ghadimi et~al., 2016]{ghadimi2016mini}
Ghadimi, S., Lan, G., and Zhang, H. (2016).
\newblock Mini-batch stochastic approximation methods for nonconvex stochastic
  composite optimization.
\newblock {\em Mathematical Programming}, 155(1-2):267--305.

\bibitem[Huang et~al., 2022]{Huang2022lower}
Huang, X., Chen, Y., Yin, W., and Yuan, K. (2022).
\newblock Lower bounds and nearly optimal algorithms in distributed learning
  with communication compression.
\newblock In Koyejo, S., Mohamed, S., Agarwal, A., Belgrave, D., Cho, K., and
  Oh, A., editors, {\em Advances in Neural Information Processing Systems},
  volume~35, pages 18955--18969. Curran Associates, Inc.

\bibitem[Jin et~al., 2017]{jin2017escape}
Jin, C., Ge, R., Netrapalli, P., Kakade, S.~M., and Jordan, M.~I. (2017).
\newblock How to escape saddle points efficiently.
\newblock In {\em Proceedings of the 34th International Conference on Machine
  Learning-Volume 70}, pages 1724--1732. JMLR. org.

\bibitem[Jin et~al., 2021]{jin2019nonconvex}
Jin, C., Netrapalli, P., Ge, R., Kakade, S.~M., and Jordan, M.~I. (2021).
\newblock On nonconvex optimization for machine learning: Gradients,
  stochasticity, and saddle points.
\newblock {\em Journal of the ACM (JACM)}, 68(2):1--29.

\bibitem[Kairouz et~al., 2019]{kairouz2019advances}
Kairouz, P., McMahan, H.~B., Avent, B., Bellet, A., Bennis, M., Bhagoji, A.~N.,
  Bonawitz, K., Charles, Z., Cormode, G., Cummings, R., et~al. (2019).
\newblock Advances and open problems in federated learning.
\newblock {\em arXiv preprint arXiv:1912.04977}.

\bibitem[Kang et~al., 2020]{kang2020reliable}
Kang, J., Xiong, Z., Niyato, D., Zou, Y., Zhang, Y., and Guizani, M. (2020).
\newblock Reliable federated learning for mobile networks.
\newblock {\em IEEE Wireless Communications}, 27(2):72--80.

\bibitem[Karimireddy et~al., 2019]{karimireddy2019error}
Karimireddy, S.~P., Rebjock, Q., Stich, S., and Jaggi, M. (2019).
\newblock Error feedback fixes signsgd and other gradient compression schemes.
\newblock In {\em International Conference on Machine Learning}, pages
  3252--3261. PMLR.

\bibitem[Koloskova et~al., 2019a]{koloskova2019decentralizeddeep}
Koloskova, A., Lin, T., Stich, S.~U., and Jaggi, M. (2019a).
\newblock Decentralized deep learning with arbitrary communication compression.
\newblock In {\em International Conference on Learning Representations}.

\bibitem[Koloskova et~al., 2020]{koloskova2020unified}
Koloskova, A., Loizou, N., Boreiri, S., Jaggi, M., and Stich, S. (2020).
\newblock A unified theory of decentralized {SGD} with changing topology and
  local updates.
\newblock In {\em International Conference on Machine Learning}, pages
  5381--5393. PMLR.

\bibitem[Koloskova et~al., 2019b]{koloskova2019decentralized}
Koloskova, A., Stich, S., and Jaggi, M. (2019b).
\newblock Decentralized stochastic optimization and gossip algorithms with
  compressed communication.
\newblock In {\em International Conference on Machine Learning}, pages
  3478--3487. PMLR.

\bibitem[Kone{\v{c}}n{\'y} et~al., 2016]{konevcny2016federatedlearning}
Kone{\v{c}}n{\'y}, J., McMahan, H.~B., Yu, F.~X., Richt{\'a}rik, P., Suresh,
  A.~T., and Bacon, D. (2016).
\newblock Federated learning: Strategies for improving communication
  efficiency.
\newblock {\em arXiv preprint arXiv:1610.05492}.

\bibitem[Krizhevsky and Hinton, 2009]{krizhevsky2009learning}
Krizhevsky, A. and Hinton, G. (2009).
\newblock Learning multiple layers of features from tiny images.

\bibitem[Lee et~al., 2016]{lee2016gradient}
Lee, J.~D., Simchowitz, M., Jordan, M.~I., and Recht, B. (2016).
\newblock Gradient descent only converges to minimizers.
\newblock In {\em 29th Annual Conference on Learning Theory}, volume~49 of {\em
  Proceedings of Machine Learning Research}, pages 1246--1257, Columbia
  University, New York, New York, USA. PMLR.

\bibitem[Li, 2019]{li2019ssrgd}
Li, Z. (2019).
\newblock {SSRGD}: Simple stochastic recursive gradient descent for escaping
  saddle points.
\newblock In {\em Advances in Neural Information Processing Systems}, pages
  1523--1533.

\bibitem[Li et~al., 2021]{li2021page}
Li, Z., Bao, H., Zhang, X., and Richt{\'a}rik, P. (2021).
\newblock {PAGE}: A simple and optimal probabilistic gradient estimator for
  nonconvex optimization.
\newblock In {\em International Conference on Machine Learning}, pages
  6286--6295. PMLR.

\bibitem[Li et~al., 2022]{li2022soteriafl}
Li, Z., Zhao, H., Li, B., and Chi, Y. (2022).
\newblock {SoteriaFL}: A unified framework for private federated learning with
  communication compression.
\newblock {\em Advances in Neural Information Processing Systems},
  35:4285--4300.

\bibitem[Nesterov, 2004]{nesterov2014introductory}
Nesterov, Y. (2004).
\newblock {\em Introductory Lectures on Convex Optimization: A Basic Course}.
\newblock Kluwer.

\bibitem[Nguyen et~al., 2021]{nguyen2021federated}
Nguyen, D.~C., Ding, M., Pathirana, P.~N., Seneviratne, A., Li, J., and Poor,
  H.~V. (2021).
\newblock Federated learning for internet of things: A comprehensive survey.
\newblock {\em IEEE Communications Surveys \& Tutorials}, 23(3):1622--1658.

\bibitem[Paszke et~al., 2019]{paszke2019pytorch}
Paszke, A., Gross, S., Massa, F., Lerer, A., Bradbury, J., Chanan, G., Killeen,
  T., Lin, Z., Gimelshein, N., and Antiga, L. (2019).
\newblock {PyTorch}: An imperative style, high-performance deep learning
  library.
\newblock In {\em Advances in {N}eural {I}nformation {P}rocessing {S}ystems},
  pages 8024--8035.

\bibitem[Richt{\'a}rik et~al., 2021]{richtarik2021ef21}
Richt{\'a}rik, P., Sokolov, I., and Fatkhullin, I. (2021).
\newblock {EF21}: A new, simpler, theoretically better, and practically faster
  error feedback.
\newblock {\em arXiv preprint arXiv:2106.05203}.

\bibitem[Richt{\'a}rik et~al., 2022]{richtarik20223pc}
Richt{\'a}rik, P., Sokolov, I., Gasanov, E., Fatkhullin, I., Li, Z., and
  Gorbunov, E. (2022).
\newblock {3PC}: Three point compressors for communication-efficient
  distributed training and a better theory for lazy aggregation.
\newblock In {\em International Conference on Machine Learning}, pages
  18596--18648. PMLR.

\bibitem[Savazzi et~al., 2020]{savazzi2020federated}
Savazzi, S., Nicoli, M., and Rampa, V. (2020).
\newblock Federated learning with cooperating devices: A consensus approach for
  massive {IoT} networks.
\newblock {\em IEEE Internet of Things Journal}, 7(5):4641--4654.

\bibitem[Seide et~al., 2014]{seide20141}
Seide, F., Fu, H., Droppo, J., Li, G., and Yu, D. (2014).
\newblock 1-bit stochastic gradient descent and its application to
  data-parallel distributed training of speech dnns.
\newblock In {\em Fifteenth annual conference of the international speech
  communication association}.

\bibitem[Stich et~al., 2018]{stich2018sparsified}
Stich, S.~U., Cordonnier, J.-B., and Jaggi, M. (2018).
\newblock Sparsified {SGD} with memory.
\newblock {\em Advances in Neural Information Processing Systems},
  31:4447--4458.

\bibitem[Xie et~al., 2020]{xie2020cser}
Xie, C., Zheng, S., Koyejo, S., Gupta, I., Li, M., and Lin, H. (2020).
\newblock Cser: Communication-efficient sgd with error reset.
\newblock In {\em Advances in Neural Information Processing Systems},
  volume~33, pages 12593--12603. Curran Associates, Inc.

\bibitem[Xu et~al., 2021]{xu2021federated}
Xu, J., Glicksberg, B.~S., Su, C., Walker, P., Bian, J., and Wang, F. (2021).
\newblock Federated learning for healthcare informatics.
\newblock {\em Journal of Healthcare Informatics Research}, 5(1):1--19.

\bibitem[Zhao et~al., 2022]{zhao2022beer}
Zhao, H., Li, B., Li, Z., Richt{\'a}rik, P., and Chi, Y. (2022).
\newblock {BEER}: Fast {$O(1/T)$} rate for decentralized nonconvex optimization
  with communication compression.
\newblock In {\em Advances in Neural Information Processing Systems}.

\end{thebibliography}
\bibliographystyle{apalike}

\appendix

\section{Technical Preparation}
\label{subsec:nSG}

Throughout, we adopt notations similar to \citet{avdiukhin2021escaping} to define several important quantities that bring convenience to our theoretical analysis. We define
\begin{enumerate}
\item local stochastic gradient noise $\bm\zeta_t^{(i)}:=\tilde{\nabla}_p f_i(\bm x_t)-\nabla f_i(\bm x_t)$,
\item local aggregate noise $\bm\psi_t^{(i)}:=\bm\zeta_t^{(i)}+\bm\xi_t$,
\item local compression error $\bm e_t^{(i)}$ as in Line \ref{line:e-update}, Algorithm \ref{alg:ALG}.
\end{enumerate}
Their global versions are defined by averaging all the nodes as
$$\bm\zeta_t:=\avg{n}\bm\zeta_t^{(i)},\quad\bm\psi_t:=\avg{n}\bm\psi_t^{(i)},\quad\bm e_t:=\avg{n}\bm e_t^{(i)}=\bm e_{t-1}+\nabla f(\bm x_{t-1})+\bm\psi_{t-1}-\bm g_{t-1}.$$

We define the sequence of corrected iterates $\{\bm y_t\}$ as $\bm y_t:=\bm x_t-\eta\bm e_t$. It is easy to verify the sequence $\{\bm y_t\}$ is updated by 
\begin{equation} \label{prop:y}
\bm y_{t+1}=\bm y_{t}-\eta(\nabla f(\bm x_t)+\bm\psi_t).
\end{equation}

Now, we introduce the definitions of norm-subGaussian random vectors and norm-subGaussian martingale difference sequences. Then we briefly state, without proof, several concentration inequalities for norm-subGaussian martingale difference sequences that underpin our theoretical derivation. Readers are referred to \citet{jin2019nonconvex} for detailed exposition. 

\begin{defi}[Definition 32, \citet{jin2019nonconvex}]\label{defi:nSG}
A random vector $\bm X\in\R^d$ is norm-subGaussian or nSG($\sigma$), if there exists $\sigma$ so that
\begin{align*}
    \Pr\left(\nm{\bm X-\EE\bm X}\geq t\right)\leq2\exp\left(-\frac{t^2}{2\sigma^2}\right)\quad\forall t\geq0.
\end{align*}
Moreover, $\bm X$ is zero-mean nSG($\sigma$) if $\EE\bm X=\bm0$ holds as well. 
\end{defi}
By definition, under Assumption \ref{ass:nSG}, each local stochastic gradient noise $\bm\zeta_t^{(i)}$ and artificial noise $\bm\xi_t$ are zero-mean nSG($\sigma$) and nSG($r$), respectively.

\begin{defi}[Condition 35, \citet{jin2019nonconvex}]\label{defi:nSG-mart-diff}
The sequence of random vectors $\bm X_1,...,\bm X_n\in\R^d$ is a norm-subGaussian martingale difference sequence with respect to the filtration $\{\mathcal F_i\}_{i=1}^n$, if $\bm X_i|\mathcal F_{i-1}$ is zero-mean nSG($\sigma_i$) for each $i\in[n]$, i.e.,
\begin{equation*}
    \EE[\bm X_i|\mathcal F_{i-1}]=\bm 0,\quad\Pr\left(\nm{\bm X_i}\geq t|\mathcal F_{i-1}\right)\leq2\exp\left(-\frac{t^2}{2\sigma_i^2}\right)\quad\forall t\geq0,i\in[n]
\end{equation*}
for some $\sigma_1,...,\sigma_n$.
\end{defi}

Regarding Algorithm \ref{alg:ALG}, a natural choice of filtration $\{\mathcal F_t\}$ is given by the $\sigma$-algebra generated by all the random variables -- all the artificial noise, stochastic gradient noise, and random operators -- up to time $t$. Now, $\{\bm\zeta_t^{(i)}\}$ and $\{\bm\xi_t\}$ are norm-subGaussian martingale difference sequences with respect to $\{\mathcal F_t\}$, due to the mutual independence between any two random variables.

In our analysis, we will make use of three concentration inequalities for such sequences.

\begin{prop}[Lemma 36, \citet{jin2019nonconvex}]\label{prop:Hoeffding}
    Let $\{\bm X_1,...,\bm X_n\}$ be a norm-subGaussian martingale difference sequence with $\sigma_1=...=\sigma_n=\sigma$. Then, there exists a constant $c$ such that for any $\iota>0$,
    \begin{equation*}
        \nm{\sum_{i=1}^n\bm X_i}^2\leq c\sigma^2n\iota
    \end{equation*}
    with probability at least $1-2de^{-\iota}$.
\end{prop}

With this, we can show that the global, accumulated stochastic gradient is a better estimator of the global true gradient, compared with each local stochastic gradient estimating its own true gradient.

\begin{coro}[Global stochastic gradient noise]\label{coro:linear-speedup}
Under Assumption \ref{ass:nSG}, there exists a constant $c$ such that the global stochastic gradient noise $\bm\zeta_t$ is a zero-mean nSG$\left(c\sigma\sqrt{\frac{\log d}{np}}\right)$ random vector.
\end{coro}
\begin{proof}
Recall that $\tilde{\nabla}_pf_i(\bm x_t)=\frac{1}{p}\sum_{j=1}^p\tilde{\nabla}f_i(\bm x_t)^{(j)}$, the average of $p$ independent stochastic gradient queries. Now, defining $\bm\zeta_i^{(i,j)}=\tilde{\nabla}f_i(\bm x_t)^{(j)}-\nabla f_i(\bm x_t)$, we have $\bm\zeta_t^{(i)}=\frac{1}{p}\sum_{j=1}^p\bm\zeta_t^{(i,j)}$, and each $\bm\zeta_t^{(i,j)}$ is zero-mean nSG($\sigma$) by Assumption \ref{ass:nSG}. Using Proposition \ref{prop:Hoeffding}, there exists some constant $c$ such that
\begin{align*}
    \Pr\left(\nm{\bm\zeta_t}^2\geq s^2\right)=\Pr\left(\nm{\sum_{i=1}^n\sum_{j=1}^p\bm\zeta_t^{(i,j)}}^2\geq n^2p^2s^2\right)\leq2d\exp\left(-\frac{nps^2}{c\sigma^2}\right)=2\exp\left(-\frac{nps^2}{c\sigma^2}+\log d\right).
\end{align*}
For $s^2\geq\frac{c\sigma^2}{np}\log2d$,
\begin{align*}
    \Pr\left(\nm{\bm\zeta_t}^2\geq s^2\right)\leq2\exp\left(-\frac{nps^2}{c\sigma^2}+\log d\right)\leq2\exp\left(-\frac{nps^2}{c\sigma^2}+\frac{nps^2}{c\sigma^2}\frac{\log d}{\log2d}\right)=2\exp\left(-\frac{nps^2}{c\sigma^2\left(1+\frac{\log d}{\log2}\right)}\right).
\end{align*}
For $s^2<\frac{c\sigma^2}{np}\log 2d$,
\begin{align*}
    \Pr\left(\nm{\bm\zeta_t}^2\geq s^2\right)\leq1<2\exp\left(-\frac{nps^2}{c\sigma^2\left(1+\frac{\log d}{\log2}\right)}\right).
\end{align*}
The above two combined establish the norm-subGaussian result.
\end{proof}

\begin{prop}[Lemma 38, \citet{jin2019nonconvex}]\label{prop:nSG-norm}
Let $\{\bm X_1,...,\bm X_n\}$ be a norm-subGaussian martingale difference sequence with $\sigma_1=...=\sigma_n=\sigma$. Then, there exists a constant $c$ such that for any $\iota>0$,
\begin{align*}
    \sum_{i=1}^n\nm{\bm X_i}^2\leq c\sigma^2(n+\iota)
\end{align*}
with probability at least $1-e^{-\iota}$.
\end{prop}

\begin{prop}[Lemma 39, \citet{jin2019nonconvex}]\label{prop:nSG-product}
Let $\{\bm X_1,...,\bm X_n\}$ be a norm-subGaussian martingale difference sequence with $\sigma_1,...,\sigma_n$, and let random vectors $\{\bm u_1,...,\bm u_n\}$ satisfy $\bm u_i\in\mathcal F_{i-1}$ for all $i\in[n]$. Then, for any $\iota>0,\lambda>0$, there exists a constant $c$ such that
\begin{align*}
    \sum_{i=1}^n\ip{\bm u_i,\bm X_i}\leq c\lambda\sum_{i=1}^n\nm{\bm u_i}^2\sigma_i^2+\lambda^{-1}\iota
\end{align*}
with probability at least $1-e^{-\iota}$.
\end{prop}
Since Algorithm \ref{alg:ALG} is iterative, $\nabla f(\bm y_t)\in\mathcal F_{t-1}$ for all $t$. This explains the validity of Proposition \ref{prop:nSG-product} when applied to our Lemma \ref{lem:comp-desc-lem} to be presented momentarily.

\section{Proof of First-order Convergence}
In this section we detail the proof of Theorem \ref{thm:FOSP}, a first-order convergence guarantee for \ALG{}. To this end, we first provide a bound for the compression error $\nm{\bm e_t}^2$ (Lemma \ref{lem:new-err-bound}), which supports an argument (Lemma \ref{lem:comp-desc-lem}) that controls the true gradient norm of the iterates produced by \ALG{}. Finally, an appropriate choice of parameters leads Lemma \ref{lem:comp-desc-lem} to the desired Theorem \ref{thm:FOSP}.
\subsection{Compression error bound}
We will use the following two lemmas to bound $\nm{\bm e_t}^2$. The first lemma controls $\nm{\bm x_{t+1}-\bm x_t}^2$, that is the difference between two consecutive iterates; the second technical lemma upper bounds a useful linear recurrence relation.
\begin{lem}\label{lem:desc}
Suppose that Assumption \ref{ass:f-L} holds, and $\eta\leq\frac{1}{2L}$. Then, the iterates $\{\bm x_t\}_{t=0}^T$ generated by Algorithm \ref{alg:ALG} satisfy
\begin{equation*}
    \nm{\bm x_{t+1}-\bm x_t}^2\leq 4\eta[f(\bm x_t)-f(\bm x_{t+1})]+4\eta^2\left(\avg{n}\nm{\bm e_{t+1}^{(i)}-\bm e_t^{(i)}}^2+\nm{\bm\psi_t}^2\right)\quad\forall t<T.
\end{equation*}
\end{lem}
\begin{proof}
Recall that Algorithm \ref{alg:ALG} updates the iterates by $\bm x_{t+1}=\bm x_t-\eta\bm g_t$ (cf. Line \ref{line:x-update}). Since $f$ is $L$-smooth under Assumption \ref{ass:f-L}, Lemma 2 of \citet{li2021page} gives
\begin{align}
    f(\bm x_{t+1})&\leq f(\bm x_t)-\frac{\eta}{2}\nm{\nabla f(\bm x_t)}^2-\left(\frac{1}{2\eta}-\frac{L}{2}\right)\nm{\bm x_{t+1}-\bm x_t}^2+\frac{\eta}{2}\nm{\nabla f(\bm x_t)-\bm g_t}^2.\nonumber
\end{align}
According to Line \ref{line:e-update} of Algorithm \ref{alg:ALG}, $\nabla f(\bm x_t)-\bm g_t=\bm e_{t+1}-\bm e_t-\bm\psi_t$. Hence, for $\eta\leq\frac{1}{2L}$,
\begin{align}
        f(\bm x_{t+1})&\leq f(\bm x_t)-\frac{\eta}{2}\nm{\nabla f(\bm x_t)}^2-\frac{1}{4\eta}\nm{\bm x_{t+1}-\bm x_t}^2+\frac{\eta}{2}\nm{\bm e_{t+1}-\bm e_t-\bm\psi_t}^2.\nonumber
\end{align}
Rearranging the terms,
\begin{align}
        \nm{\bm x_{t+1}-\bm x_t}^2&\leq4\eta[f(\bm x_{t})-f(\bm x_{t+1})]+2\eta^2\nm{\bm e_{t+1}-\bm e_t-\bm\psi_t}^2-2\eta^2\nm{\nabla f(\bm x_t)}^2\nonumber\\
        &\leq4\eta[f(\bm x_{t})-f(\bm x_{t+1})]+4\eta^2\left(\nm{\bm e_{t+1}-\bm e_t}^2+\nm{\bm\psi_t}^2\right)\nonumber\\
        &\leq 4\eta[f(\bm x_t)-f(\bm x_{t+1})]+4\eta^2\left(\avg{n}\nm{\bm e_{t+1}^{(i)}-\bm e_t^{(i)}}^2+\nm{\bm\psi_t}^2\right),\nonumber
\end{align}
where the last step is due to Jensen's inequality.
\end{proof}

\begin{lem}[Sum of compression error]\label{lem:new-err-bound}
    Suppose that Assumption \ref{ass:f-min}, \ref{ass:f-L}, \ref{ass:stoch-L}, \ref{ass:nSG} hold, and $\eta$, $p$ satisfy
    \begin{equation*}
        \eta\leq\min\left\{\frac{\mu}{24\tilde{L}},\frac{1}{2L}\right\},\quad p\geq\frac{\log(\mu^2/144)}{\log(1-\mu)}.
    \end{equation*}
    Fix any $t\leq T$. Then, there exists a constant $c$, such that the sum of compression error produced by Algorithm \ref{alg:ALG} prior to iteration $t$ is bounded by
    $$
    \sum_{\tau=0}^{t-1}\nm{\bm e_\tau}^2\leq c\left[\frac{\Phi}{\mu}+\frac{\chi^2(t+\iota)}{\mu^2np}\right]\quad\forall \tau\leq t
    $$
    with probability at least $1-3e^{-\iota}$.
\end{lem}
\begin{proof}
By Line \ref{line:c-update}, \ref{line:g-update}, \ref{line:e-update} of Algorithm \ref{alg:ALG},
\begin{align}
    \bm e_{\tau+1}^{(i)}=\bm e_\tau^{(i)}+\tilde{\nabla}_p f_i(\bm x_\tau)+\bm\xi_\tau-\bm g_{\tau-1}^{(i)}-\bm w_\tau^{(i)}-\mathcal{C}\left(\bm e_\tau^{(i)}+\tilde{\nabla}_p f_i(\bm x_\tau)+\bm\xi_\tau-\bm g_{\tau-1}^{(i)}-\bm w_\tau^{(i)}\right).\nonumber
\end{align}
Hence,
\begin{align}
    \nm{\bm e_{\tau+1}^{(i)}}^2&=\nm{\bm e_\tau^{(i)}+\tilde{\nabla}_p f_i(\bm x_\tau)+\bm\xi_\tau-\bm g_{\tau-1}^{(i)}-\bm w_\tau^{(i)}-\mathcal{C}\left(\bm e_\tau^{(i)}+\tilde{\nabla}_p f_i(\bm x_\tau)+\bm\xi_\tau-\bm g_{\tau-1}^{(i)}-\bm w_\tau^{(i)}\right)}^2\nonumber\\
    &\leq(1-\mu)\nm{\bm e_\tau^{(i)}+\tilde{\nabla}_p f_i(\bm x_\tau)+\bm\xi_\tau-\bm g_{\tau-1}^{(i)}-\bm w_\tau^{(i)}}^2\label{eq:err-bound-0}\\
    &\leq(1-\mu)(1+\nu)\nm{\bm e_\tau^{(i)}}^2+(1-\mu)(1+\nu^{-1})\nm{\tilde{\nabla}_p f_i(\bm x_\tau)+\bm\xi_\tau-\bm g_{\tau-1}^{(i)}-\bm w_\tau^{(i)}}^2, \label{eq:err-bound-1}
\end{align}
where \eqref{eq:err-bound-0} is due to the compression property of $\mathcal C$ (cf. Definition \ref{defi:C}), and we invoke Young's inequality with arbitrary $\nu>0$ in \eqref{eq:err-bound-1}. Moreover, note the identity
\begin{align}
    &\tilde{\nabla}_p f_i(\bm x_\tau)+\bm\xi_\tau-\bm g_{\tau-1}^{(i)}-\bm w_\tau^{(i)}\nonumber\\
   & =\tilde{\nabla}_p f_i(\bm x_\tau)-\tilde{\nabla}_p f_i(\bm x_{\tau-1})+\bm\xi_\tau-\bm\xi_{\tau-1}+(\tilde{\nabla}_p f_i(\bm x_{\tau-1})+\bm\xi_{\tau-1}-\bm g_{\tau-1}^{(i)})-\bm w_\tau^{(i)}\nonumber\\
  &  =\tilde{\nabla}_p f_i(\bm x_\tau)-\tilde{\nabla}_p f_i(\bm x_{\tau-1})+\bm\xi_\tau-\bm\xi_{\tau-1}+\bm e_\tau^{(i)}-\bm e_{\tau-1}^{(i)}-\bm w_\tau^{(i)}\label{eq:err-bound-id-1}\\
   & =\tilde{\nabla}_p f_i(\bm x_\tau)-\tilde{\nabla}_p f_i(\bm x_{\tau-1})+\bm\xi_\tau-\bm\xi_{\tau-1}+\mathcal D^p(\bm e_\tau^{(i)}-\bm e_{\tau-1}^{(i)}),\label{eq:err-bound-id-2}
\end{align}
where Line \ref{line:e-update} and \ref{line:w-update} of Algorithm \ref{alg:ALG} imply \eqref{eq:err-bound-id-1} and \eqref{eq:err-bound-id-2}, respectively. Plugging \eqref{eq:err-bound-id-2} into \eqref{eq:err-bound-1} yields
\begin{align}
    \nm{\bm e_{\tau+1}^{(i)}}^2&=(1-\mu)(1+\nu)\nm{\bm e_\tau^{(i)}}^2+(1-\mu)(1+\nu^{-1})\nm{\tilde{\nabla}_p f_i(\bm x_\tau)-\tilde{\nabla}_p f_i(\bm x_{\tau-1})+\bm\xi_\tau-\bm\xi_{\tau-1}+\mathcal{D}^p(\bm e_\tau^{(i)}-\bm e_{\tau-1}^{(i)})}^2.\nonumber
\end{align}
Now, simply take $\nu=\frac{\mu}{2(1-\mu)}$,
\begin{align}
    \nm{\bm e_{\tau+1}^{(i)}}^2&\leq\left(1-\frac{\mu}{2}\right)\nm{\bm e_\tau^{(i)}}^2+\frac{6}{\mu}\left(\nm{\mathcal{D}^p(\bm e_\tau^{(i)}-\bm e_{\tau-1}^{(i)})}^2+\nm{\tilde{\nabla}_p f_i(\bm x_\tau)-\tilde{\nabla}_p f_i(\bm x_{\tau-1})}^2+\nm{\bm\xi_\tau-\bm\xi_{\tau-1}}^2\right)\nonumber\\
    &\leq\left(1-\frac{\mu}{2}\right)\nm{\bm e_\tau^{(i)}}^2+\frac{6}{\mu}\left((1-\mu)^p\nm{\bm e_\tau^{(i)}-\bm e_{\tau-1}^{(i)}}^2+\tilde L_i^2\nm{\bm x_\tau-\bm x_{\tau-1}}^2+\nm{\bm\xi_\tau-\bm\xi_{\tau-1}}^2\right),\label{eq:eti-bound}
\end{align}
where \eqref{eq:eti-bound} follows from the contraction property of operator $\mathcal D^p$ and Lipschitz property of $\tilde\nabla f_i$ in Assumption \ref{ass:stoch-L}. Now, averaging \eqref{eq:eti-bound} over all the nodes and setting $Q_\tau=\avg{n}\nm{\bm e_\tau^{(i)}}^2$,
\begin{align}
    Q_{\tau+1}&\leq\left(1-\frac{\mu}{2}\right)Q_\tau+\frac{6}{\mu}\Bigg((1-\mu)^p\avg{n}\nm{\bm e_\tau^{(i)}-\bm e_{\tau-1}^{(i)}}^2+\tilde L^2\nm{\bm x_\tau-\bm x_{\tau-1}}^2+\nm{\bm\xi_\tau-\bm\xi_{\tau-1}}^2\Bigg)\nonumber\\
    &\leq\left(1-\frac{\mu}{2}\right)Q_\tau+\frac{6}{\mu}\Bigg[\left((1-\mu)^p+4\tilde L^2\eta^2\right)\avg{n}\nm{\bm e_\tau^{(i)}-\bm e_{\tau-1}^{(i)}}^2\nonumber\\
    &\eqbr+4\tilde L^2\eta[f(\bm x_{\tau-1})-f(\bm x_\tau)]+4\tilde L^2\eta^2\nm{\bm\psi_{\tau-1}}^2+\nm{\bm\xi_\tau-\bm\xi_{\tau-1}}^2\Bigg]\label{eq:avg-eti-bound-1}.
\end{align}
Here, we obtain \eqref{eq:avg-eti-bound-1} as a direct consequence of Lemma \ref{lem:desc}. Due to our choice of $\eta$ and $p$, we can proceed from \eqref{eq:avg-eti-bound-1} to
\begin{align}
    Q_{\tau+1}&\leq\left(1-\frac{\mu}{2}\right)Q_\tau+\frac{\mu}{12}\avg{n}\nm{\bm e_\tau^{(i)}-\bm e_{\tau-1}^{(i)}}^2+\tilde L[f(\bm x_{\tau-1})-f(\bm x_\tau)]+\frac{\mu}{24}\nm{\bm\psi_{\tau-1}}^2+\frac{6}{\mu}\nm{\bm\xi_\tau-\bm\xi_{\tau-1}}^2\nonumber\\
    &\leq\left(1-\frac{\mu}{3}\right)Q_\tau+\frac{\mu}{6}Q_{\tau-1}+\tilde L[f(\bm x_{\tau-1})-f(\bm x_\tau)]+\frac{\mu}{24}\nm{\bm\psi_{\tau-1}}^2+\frac{6}{\mu}\nm{\bm\xi_\tau-\bm\xi_{\tau-1}}^2.\label{eq:avg-eti-bound}
\end{align}
Applying \eqref{eq:avg-eti-bound} for $\tau=1,2,...,t-2$ respectively, we have
\begin{align*}
    \sum_{\tau=0}^{t-1}Q_\tau&\leq Q_0+Q_1+\left(1-\frac{\mu}{3}\right)\sum_{\tau=1}^{t-2}Q_\tau+\frac{\mu}{6}\sum_{\tau=0}^{t-3}Q_\tau+\tilde L[f(\bm x_{0})-f(\bm x_{t-2})]+\frac{\mu}{24}\sum_{\tau=0}^{t-3}\nm{\bm\psi_\tau}^2+\frac{6}{\mu}\sum_{\tau=1}^{t-2}\nm{\bm\xi_\tau-\bm\xi_{\tau-1}}^2\nonumber\\
    &\leq Q_0+Q_1+\left(1-\frac{\mu}{3}\right)\sum_{\tau=0}^{t-1}Q_\tau+\frac{\mu}{6}\sum_{\tau=0}^{t-1}Q_\tau+\tilde L[f(\bm x_{0})-f_{\min}]+\frac{\mu}{24}\sum_{\tau=0}^{t-1}\nm{\bm\psi_\tau}^2+\frac{6}{\mu}\sum_{\tau=0}^{t-1}\nm{\bm\xi_\tau-\bm\xi_{\tau-1}}^2,\nonumber
\end{align*}
where the last step uses the non-negativity of terms and Assumption \ref{ass:f-min}. After rearranging,
\begin{align*}
    \sum_{\tau=0}^{t-1}Q_\tau&\leq\frac{6}{\mu}\left(Q_0+Q_1+\tilde L[f(\bm x_0)-f_{\min}]\right)+\frac{1}{4}\sum_{\tau=0}^{t-1}\nm{\bm\psi_\tau}^2+\frac{36}{\mu^2}\sum_{\tau=0}^{t-1}\nm{\bm\xi_\tau-\bm\xi_{\tau-1}}^2\\
    &\leq \frac{6}{\mu}\left(Q_0+Q_1+\tilde L[f(\bm x_0)-f_{\min}]\right)+\frac{1}{2}\sum_{\tau=0}^{t-1}\nm{\bm\zeta_\tau}^2+\frac{1}{2}\sum_{\tau=0}^{t-1}\nm{\bm\xi_\tau}^2+\frac{36}{\mu^2}\sum_{\tau=0}^{t-1}\nm{\bm\xi_\tau-\bm\xi_{\tau-1}}^2.
\end{align*}
The first term is bounded by $6\Phi/\mu$, as one can verify that
\begin{align*}
    Q_0=0;\quad Q_1\leq\avg{n}\nm{\tilde\nabla_p f_i(\bm x_0)+\bm\xi_0}^2.
\end{align*}
Moreover, by Corollary \ref{coro:linear-speedup} as well as Proposition \ref{prop:nSG-norm}, with probability at least $1-3e^{-\iota}$, there exists a constant $c$ such that
\begin{align*}
    \frac{1}{2}\sum_{\tau=0}^{t-1}\nm{\bm\zeta_\tau}^2+\frac{1}{2}\sum_{\tau=0}^{t-1}\nm{\bm\xi_\tau}^2+\frac{36}{\mu^2}\sum_{\tau=0}^{t-1}\nm{\bm\xi_\tau-\bm\xi_{\tau-1}}^2\leq\frac{c}{\mu^2}\left(\frac{\sigma^2\log d+r^2}{np}\right)(t+\iota)=\frac{c\chi^2(t+\iota)}{\mu^2np}.
\end{align*}
This completes the proof.
\end{proof}

\subsection{Convergence}
Lemma \ref{lem:err-bound} results in the following argument, which is essential for showing the first-order convergence.
\begin{lem}[Descent lemma]\label{lem:comp-desc-lem}
    Suppose that Assumption \ref{ass:f-min}, \ref{ass:f-L}, \ref{ass:nSG} hold, and $\eta$, $p$ satisfy
    \begin{equation*}
        \eta\leq\min\left\{\frac{\mu}{24\tilde{L}},\frac{1}{12L}\right\}\quad p\geq\frac{\log(\mu^2/144)}{\log(1-\mu)}.
    \end{equation*}
    Then there exists some constant $c$ such that for any $t\leq T$, 
    \begin{equation*}
        \sum_{\tau=0}^{t-1}\nm{\nabla f(\bm x_\tau)}^2\leq\frac{8[f(\bm y_0)-f(\bm y_t)]}{\eta}+c\eta^2L^2\left(\frac{\Phi}{\mu}+\frac{\chi^2\iota}{\mu^2np}\right)+c(\eta L+1)\frac{\chi^2\iota}{np}+c\eta L\left(\frac{\eta L}{\mu^2}+1\right)\frac{\chi^2 T}{np}
    \end{equation*}
    with probability at least $1-7e^{-\iota}$.
\end{lem}
\begin{proof}
Under Assumption \ref{ass:f-L}, the $L$-smoothness of $f$ implies
\begin{align}
    f(\bm y_{t+1})&\leq f(\bm y_t)+\ip{\nabla f(\bm y_t),\bm y_{t+1}-\bm y_t}+\frac{L}{2}\nm{\bm y_{t+1}-\bm y_t}^2\nonumber\\
    &=f(\bm y_t)-\eta\ip{\nabla f(\bm y_t),\nabla f(\bm x_t)+\bm\psi_t}+\frac{L\eta^2}{2}\nm{\nabla f(\bm x_t)+\bm\psi_t}^2\label{eq:comp-desc-1}\\
    &\leq f(\bm y_t)-\eta\nm{\nabla f(\bm x_t)}^2-\eta\ip{\nabla f(\bm y_t)-\nabla f(\bm x_t),\nabla f(\bm x_t)}-\eta\ip{\nabla f(\bm y_t),\bm\zeta_t}-\eta\ip{\nabla f(\bm y_t),\bm\xi_t}\nonumber\\
    &\eqbr+\frac{3L\eta^2}{2}\left(\nm{\nabla f(\bm x_t)}^2+\nm{\bm\zeta_t}^2+\nm{\bm\xi_t}^2\right)\nonumber\\
    &\leq f(\bm y_t)-\eta\nm{\nabla f(\bm x_t)}^2+\frac{\eta}{2}\nm{\nabla f(\bm y_t)-\nabla f(\bm x_t)}^2+\frac{\eta}{2}\nm{\nabla f(\bm x_t)}^2-\eta\ip{\nabla f(\bm y_t),\bm\zeta_t}-\eta\ip{\nabla f(\bm y_t),\bm\xi_t}\nonumber\\
    &\eqbr+\frac{3L\eta^2}{2}\left(\nm{\nabla f(\bm x_t)}^2+\nm{\bm\zeta_t}^2+\nm{\bm\xi_t}^2\right),\nonumber
\end{align}
where \eqref{eq:comp-desc-1} is due to Proposition \ref{prop:y}. Sum up and rearrange the terms, we have
\begin{align}
    f(\bm y_T)-f(\bm y_0)&\leq-\eta\left(\frac{1}{2}-\frac{3L\eta}{2}\right)\sum_{\tau=0}^{T-1}\nm{\nabla f(\bm x_\tau)}^2+\frac{\eta}{2}\sum_{\tau=0}^{T-1}\nm{\nabla f(\bm y_\tau)-\nabla f(\bm x_\tau)}^2\nonumber\\
    &\eqbr+\frac{3L\eta^2}{2}\sum_{\tau=0}^{T-1}\left(\nm{\bm\zeta_\tau}^2+\nm{\bm\xi_\tau}^2\right)-\eta\sum_{\tau=0}^{T-1}\ip{\nabla f(\bm y_\tau),\bm\zeta_\tau}-\eta\sum_{\tau=0}^{T-1}\ip{\nabla f(\bm y_\tau),\bm\xi_\tau}.\label{eq:desc-1}
\end{align}
According to Proposition \ref{prop:nSG-norm} and \ref{prop:nSG-product} as well as union bound, there exist constants $c_1$ and $c_2$ such that
\begin{equation}
    \frac{3L\eta^2}{2}\left(\sum_{\tau=0}^{T-1}\nm{\bm\zeta_\tau}^2+\sum_{\tau=0}^{T-1}\nm{\bm\xi_\tau}^2\right)\leq c_1\frac{L\eta^2\chi^2(T+\iota)}{np}\label{eq:desc-2}
\end{equation}
and
\begin{align}
    -\eta\left(\sum_{\tau=0}^{T-1}\ip{\nabla f(\bm y_\tau),\bm\zeta_\tau}
    +\sum_{\tau=0}^{T-1}\ip{\nabla f(\bm y_\tau),\bm\xi_\tau}\right)&\leq\frac{\eta}{8}\sum_{\tau=0}^{T-1}\nm{\nabla f(\bm y_\tau)}^2+c_2\frac{\eta\chi^2\iota}{np}\label{eq:desc-3}
\end{align}
hold simultaneously with probability at least $1-4e^{-\iota}$. Plugging \eqref{eq:desc-2} and \eqref{eq:desc-3} back into \eqref{eq:desc-1} gives
\begin{align}
    f(\bm y_T)-f(\bm y_0)&\leq-\eta\left(\frac{1}{2}-\frac{3L\eta}{2}\right)\sum_{\tau=0}^{T-1}\nm{\nabla f(\bm x_\tau)}^2+\frac{\eta}{2}\sum_{\tau=0}^{T-1}\nm{\nabla f(\bm y_\tau)-\nabla f(\bm x_\tau)}^2\nonumber\\
    &\eqbr+\frac{\eta}{8}\sum_{\tau=0}^{T-1}\nm{\nabla f(\bm y_\tau)}^2+c_1\frac{L\eta^2\chi^2(T+\iota)}{np}+c_2\frac{\eta\chi^2\iota}{np}\nonumber\\
    &\leq-\eta\left(\frac{1}{2}-\frac{3L\eta}{2}\right)\sum_{\tau=0}^{T-1}\nm{\nabla f(\bm x_\tau)}^2+\frac{\eta}{2}\sum_{\tau=0}^{T-1}\nm{\nabla f(\bm y_\tau)-\nabla f(\bm x_\tau)}^2\nonumber\\
    &\eqbr+\frac{\eta}{4}\sum_{\tau=0}^{T-1}\left(\nm{\nabla f(\bm y_\tau)-\nabla f(\bm x_\tau)}^2+\nm{\nabla f(\bm x_\tau)}^2\right)+c_1\frac{L\eta^2\chi^2(T+\iota)}{np}+c_2\frac{\eta\chi^2\iota}{np}\nonumber\\
    &=-\eta\left(\frac{1}{4}-\frac{3L\eta}{2}\right)\sum_{\tau=0}^{T-1}\nm{\nabla f(\bm x_\tau)}^2+\frac{3\eta}{4}\sum_{\tau=0}^{T-1}\nm{\nabla f(\bm y_\tau)-\nabla f(\bm x_\tau)}^2\nonumber\\
    &\eqbr+c_1\frac{L\eta^2\chi^2(T+\iota)}{np}+c_2\frac{\eta\chi^2\iota}{np}\nonumber\\
    &\leq-\eta\left(\frac{1}{4}-\frac{3L\eta}{2}\right)\sum_{\tau=0}^{T-1}\nm{\nabla f(\bm x_\tau)}^2+\frac{3\eta L^2}{4}\sum_{\tau=0}^{T-1}\nm{\bm y_\tau-\bm x_\tau}^2+c_1\frac{L\eta^2\chi^2(T+\iota)}{np}+c_2\frac{\eta\chi^2\iota}{np}\label{eq:comp-desc-2}\\
    &=-\eta\left(\frac{1}{4}-\frac{3L\eta}{2}\right)\sum_{\tau=0}^{T-1}\nm{\nabla f(\bm x_\tau)}^2+\frac{3\eta^3L^2}{4}\sum_{\tau=0}^{T-1}\nm{\bm e_\tau}^2+c_1\frac{L\eta^2\chi^2(T+\iota)}{np}+c_2\frac{\eta\chi^2\iota}{np}\nonumber\\
    &\leq-\frac{\eta}{8}\sum_{\tau=0}^{T-1}\nm{\nabla f(\bm x_\tau)}^2+\frac{3\eta^3L^2}{4}\sum_{\tau=0}^{T-1}\nm{\bm e_\tau}^2+c_1\frac{L\eta^2\chi^2(T+\iota)}{np}+c_2\frac{\eta\chi^2\iota}{np}\label{eq:comp-desc-3}
\end{align}
with probability at least $1-4e^{-\iota}$. In the above derivation, we make use of $L$-smoothness of $f$ in \eqref{eq:comp-desc-2} and our appropriate choice of $\eta$ in \eqref{eq:comp-desc-3}. Finally, by Lemma \ref{lem:new-err-bound} and union bound, with probability at least $1-7e^{-\iota}$ we have
\begin{align}
    f(\bm y_T)-f(\bm y_0)&\leq-\frac{\eta}{8}\sum_{\tau=0}^{T-1}\nm{\nabla f(\bm x_\tau)}^2+\frac{3c_3\eta^3 L^2}{4}\left(\frac{\Phi}{\mu}+\frac{\chi^2(T+\iota)}{\mu^2np}\right)+c_1\frac{L\eta^2\chi^2(T+\iota)}{np}+c_2\frac{\eta\chi^2\iota}{np}\nonumber\\
    \sum_{\tau=0}^{T-1}\nm{\nabla f(\bm x_\tau)}^2&\leq\frac{8[f(\bm y_0)-f(\bm y_t)]}{\eta}+6c_3\eta^2L^2\left(\frac{\Phi}{\mu}+\frac{\chi^2\iota}{\mu^2np}\right)+8(c_1\eta L+c_2)\frac{\chi^2\iota}{np}+\left(\frac{6c_3\eta^2L^2}{\mu^2}+8c_1\eta L\right)\frac{\chi^2 T}{np}\nonumber\\
    &\leq\frac{8[f(\bm y_0)-f(\bm y_t)]}{\eta}+c\eta^2L^2\left(\frac{\Phi}{\mu}+\frac{\chi^2\iota}{\mu^2np}\right)+c(\eta L+1)\frac{\chi^2\iota}{np}+c\eta L\left(\frac{\eta L}{\mu^2}+1\right)\frac{\chi^2 T}{np}\nonumber
\end{align}
for an appropriate constant $c$.
\end{proof}

We are now ready to establish the desired result regarding the convergence to $\epsilon$-FOSPs.

\begin{proof}[Proof of Theorem \ref{thm:FOSP}]
Otherwise, at least a quarter of the iterates have gradient norm larger than $\epsilon$. Hence
\begin{equation*}
    \sum_{\tau=0}^{T-1}\nm{\nabla f(\bm x_\tau)}^2>\frac{T}{4}\epsilon^2.
\end{equation*}
However, taking our choice of $\eta$ and $T$ into Lemma \ref{lem:comp-desc-lem}, the following holds with probability at least $1-7e^{-\iota}$:
\begin{equation}
    \sum_{\tau=0}^{T-1}\nm{\nabla f(\bm x_\tau)}^2\leq T\epsilon^2\left(\frac{8}{\kappa_T}+2c\kappa_\eta^2+2c\kappa_\eta+\frac{c}{\kappa_T}\right).\label{eq:comp-desc-4}
\end{equation}
When we set $\kappa_T\geq8(c+8)$ and $\kappa_\eta\leq\frac{1}{32c}$, \eqref{eq:comp-desc-4} implies $\sum_{\tau=0}^{T-1}\nm{\nabla f(\bm x_\tau)}^2\leq T\epsilon^2/4$, which produces a contradiction.
\end{proof}

\section{Proof of Second-order Convergence}
The core idea for establishing the second-order convergence result (Theorem \ref{thm:SOSP}) is to show that, when \ALG{} encounters a saddle point, the objective can still descend sufficiently after finitely many additional iterations (Lemma \ref{lem:saddle-desc}). 

Two arguments are developed to support this favorable property of \ALG{} dynamics. Firstly, we show an improve-or-localize behavior of \ALG{} (Lemma \ref{lem:imp-or-loc}): if the iterates $\{\bm y_t\}$ escape (move far enough) from a saddle point, the objective must descend sufficiently. 

Secondly, we claim that the iterates do escape from saddle points (Corollary \ref{coro:escape-saddle}). This nontrivial claim is obtained using the coupling sequences technique. To be specific, we craft another sequence $\{\bm y'_t\}$ mirroring the original iterates $\{\bm y_t\}$ along the escape direction of a saddle point. We show that the gap between the coupling sequences $\nm{\bm y_t-\bm y'_t}$ expands sufficiently after finitely many iterations (Lemma \ref{lem:coup-seq-dyn}), which implies $\{\bm y_t\}$ travels far from the saddle point.

This workflow of establishing second-order convergence guarantees stems from pioneering works on plain GD and SGD \citep{jin2017escape,jin2019nonconvex}, and  finds similar applications in several recent works such as recursive SGD \citep{li2019ssrgd}, and compressed SGD \citep{avdiukhin2021escaping}. While the theory in \citet{avdiukhin2021escaping} entails respective discussions on the large-gradient case and small-gradient case, \ALG{} avoids such intricacies due to the technical fact that our bound for $\bm e_t$ does not involve gradient norm terms.

For conciseness, we presume the following parameter setting for our theory and do not restate them therein.
\begin{align}
    r&=\kappa_r\sigma\sqrt{\iota d\log d}, \nonumber\\
    \eta&=\kappa_\eta\cdot\min\left\{\frac{\mu\epsilon}{\iota^5L\sqrt{\mu\Phi+\frac{\chi^2\iota}{np}}},\frac{\iota\sigma^2\sqrt{\rho\epsilon}\log d}{L^2(np\Phi+\frac{\chi^2\iota}{\mu^2})},\frac{np\epsilon^2}{\iota^5L\chi^2}\right\}, \nonumber\\
    T&=\kappa_T\cdot\max\left\{\frac{\iota^5f_{\max}}{\eta\epsilon^2},\frac{\chi^2\iota}{np\epsilon^2}\right\},  \nonumber\\
    \mathcal{I}&=\frac{\iota}{\eta\sqrt{\rho\epsilon}}, \nonumber\\
    \mathcal{R}&=\kappa_\mathcal{R}\sqrt{\frac{\epsilon}{\iota^3\rho}},  \nonumber\\
    \mathcal{F}&=\frac{\kappa_\mathcal{F}}{\iota^4}\sqrt{\frac{\epsilon^3}{\rho}}.  \nonumber
\end{align}
Here, $\kappa_r,\kappa_\eta,\kappa_T,\kappa_\mathcal{R},\kappa_\mathcal{F}$ are numerical constants to be determined in the detailed proofs.

\subsection{Uniform error bound}
With the aid of Assumption \ref{ass:f-max}, we can develop a strengthened error bound that not only controls the sum of compression errors Lemma \ref{lem:new-err-bound}, but also uniformly controls each individual error term. We begin with a technical result regarding a recurrence relation.
\begin{lem}\label{lem:recurrence}
    Consider a real sequence $\{r_t\}$ such that $r_{t+1}\leq Ar_t+Br_{t-1}+C$ for some positive constants $A,B,C$ and the initial values $r_0=0,r_1\geq 0$. If $A+B<1$, then for any $t\geq0$ we have
    \begin{equation*}
        r_t\leq\frac{2r_1}{A}+\frac{6C}{A(1-A-B)}.
    \end{equation*}
\end{lem}
\begin{proof}
    Consider another real sequence $\{p_t\}$ with $p_{t+1}= Ap_t+Bp_{t-1}+C$ and $p_0=0,p_1=r_1$. Clearly, $r_t\leq p_t$. Solving the recurrence about $\{p_t\}$ yields
    \begin{align}
        p_t&=c_1\left(\frac{A-\sqrt{A^2+4B}}{2}\right)^t+c_2\left(\frac{A+\sqrt{A^2+4B}}{2}\right)^t+\frac{C}{1-A-B}\nonumber\\
        &\leq(|c_1|+|c_2|)\left(\frac{A+\sqrt{A^2+4B}}{2}\right)^t+\frac{C}{1-A-B},\nonumber
    \end{align}
    where $c_1,c_2$ are determined by the initial values. Since $A+B<1$, we have $\frac{A+\sqrt{A^2+4B}}{2}<1$, hence $p_t\leq|c_1|+|c_2|+\frac{C}{1-A-B}$ for any $t\geq0$. It remains to bound $|c_1|$ and $|c_2|$, which is straightforward.
\end{proof}

Now we have the following lemma.

\begin{lem}[Uniform error bound]\label{lem:err-bound}
    Suppose that Assumption \ref{ass:f-max}, \ref{ass:f-L}, \ref{ass:nSG} hold, and $\eta$, $p$ satisfy
    \begin{equation*}
        \eta\leq\min\left\{\frac{\mu}{24L},\frac{\chi^2\iota}{4npL^2f_{\max}}\right\},\quad p\geq\frac{\log(\mu^2/144)}{\log(1-\mu)}.
    \end{equation*}
    Fix any $t\leq T$. Then, there exists a constant $c$ such that the compression error terms produced by Algorithm \ref{alg:ALG} prior to iteration $t$ are uniformly bounded by
    $$\nm{\bm e_\tau}^2\leq c\left(\Phi+\frac{\chi^2\iota}{\mu^2np}\right)\quad\forall \tau\leq t$$
    with probability at least $1-6te^{-\iota}$.
\end{lem}
\begin{proof}
Starting from \eqref{eq:avg-eti-bound-1}, by $|f(\bm x_\tau)-f(\bm x_{\tau-1})|<f_{\max}$ and the norm-subGaussian properties of $\bm\psi_{\tau}$ and $\bm\xi_\tau$, we have
\begin{align}
    Q_{\tau+1}\leq\left(1-\frac{\mu}{3}\right)Q_\tau+\frac{\mu}{6}Q_{\tau-1}+\frac{24\tilde L^2\eta f_{\max}}{\mu}+\frac{24c(\tilde L^2\eta^2+1)}{\mu}\frac{\chi^2\iota}{np}\nonumber
\end{align}
for some constant $c_1$, with probability at least $1-6e^{-\iota}$. Due to our choice of $\eta$, there exists some constant $c_2$ such that
\begin{align}
    Q_{\tau+1}\leq\left(1-\frac{\mu}{3}\right)Q_\tau+\frac{\mu}{6}Q_{\tau-1}+\frac{c_2\chi^2\iota}{\mu np}.\label{eq:strengthened-err-bound-recurrence}
\end{align}
By union bound, \eqref{eq:strengthened-err-bound-recurrence} holds for all $\tau<t$ with probability $1-6te^{-\iota}$, thus a recurrence relation taking the form in Lemma \ref{lem:recurrence} with
\begin{equation*}
    Q_0=0,\quad Q_1\leq(1-\mu)\avg{n}\nm{\nabla f_i(\bm x_0)+\bm\psi_0^{(i)}}^2\leq\Phi.
\end{equation*}
Following Lemma \ref{lem:recurrence}, for all $\tau\leq t$
\begin{align*}
    \nm{\bm e_\tau}^2&\leq Q_\tau\leq c\left(\Phi+\frac{\chi^2\iota}{\mu^2np}\right)
\end{align*}
for some constant $c$, which completes the proof.
\end{proof}

\subsection{Improve-or-localize behavior}

\begin{lem}[Improve or localize]\label{lem:imp-or-loc}
    Suppose that Assumption \ref{ass:f-max}, \ref{ass:f-L}, \ref{ass:nSG} hold. Let $t_0$ and $t$ be given arbitrarily. There exists a constant $c$ such that with probability at least $1-7te^{-\iota}$,
    \begin{equation*}
        f(\bm y_{t_0})-f(\bm y_{t_0+t})\geq\frac{1}{16\eta t}\cdot\max_{\tau\leq t}\nm{\bm y_{t_0+\tau}-\bm y_{t_0}}^2-c\kappa_\eta\epsilon^2(\eta t+\iota).
    \end{equation*}
\end{lem}
\begin{proof}
For any $\tau\leq t$,
\begin{align}
    \nm{\bm y_{t_0+\tau}-\bm y_{t_0}}^2&=\nm{\sum_{j=0}^{\tau-1}\left(\bm y_{t_0+j+1}-\bm y_{t_0+j}\right)}^2\nonumber\\
    &=\eta^2\nm{\sum_{j=0}^{\tau-1}[\nabla f(\bm x_{t_0+j})+\bm\psi_{t_0+j}]}^2\label{eq:imp-or-loc-1}\\
    &\leq 2\eta^2\tau\sum_{j=0}^{\tau-1}\nm{\nabla f(\bm x_{t_0+j})}^2+2\eta^2\nm{\sum_{j=0}^{\tau-1}\bm\psi_{t_0+j}}^2\leq 2\eta^2t\sum_{j=0}^{t-1}\nm{\nabla f(\bm x_{t_0+j})}^2+2\eta^2\nm{\sum_{j=0}^{t-1}\bm\psi_{t_0+j}}^2\nonumber\\
    &\leq2\eta^2t\bigg[\frac{8[f(\bm y_{t_0})-f(\bm y_{t_0+t})]}{\eta}+6c\eta^2L^2\left(\frac{\Phi}{\mu}+\frac{\chi^2\iota}{\mu^2np}\right)+8(c_1\eta L+c_2)\frac{\chi^2\iota}{np}\nonumber\\
    &\eqbr+\left(\frac{6c\eta^2L^2}{\mu^2}+8c_1\eta L\right)\frac{\chi^2t}{np}\bigg]+\frac{2\eta^2c_1\chi^2(t+\iota)}{np}\label{eq:imp-or-loc-2}
\end{align}
with probability at least $1-7e^{-\iota}$, as we invoke Proposition \ref{prop:y} in \eqref{eq:imp-or-loc-1} and Lemma \ref{lem:comp-desc-lem} in \eqref{eq:imp-or-loc-2}. Rearranging the terms, for some constant $c$ we have
    \begin{align}
        f(\bm y_{t_0})-f(\bm y_{t_0+t})
        &\geq\frac{\nm{\bm y_{t_0+\tau}-\bm y_{t_0}}^2}{16\eta t}-\left[c\eta t\frac{\chi^2}{np}\left(\frac{\eta^2L^2}{\mu^2}+\eta L\right)+c\eta^3L^2\left(\frac{\Phi}{\mu}+\frac{\chi^2\iota}{\mu^2np}\right)+c\eta\frac{\chi^2\iota}{np}\right]\nonumber\\
        &\geq\frac{\nm{\bm y_{t_0+\tau}-\bm y_{t_0}}^2}{16\eta t}-\left[c\eta t\left(\frac{\kappa_\eta}{\iota^{10}}\epsilon^2+\frac{\kappa_\eta^2}{\iota^{5}}\epsilon^2\right)+c\eta\kappa_\eta\frac{\epsilon^2}{\iota^5}+c\kappa_\eta\frac{\epsilon^2}{L\iota^5}\iota\right]\nonumber\\
        &\geq \frac{\nm{\bm y_{t_0+\tau}-\bm y_{t_0}}^2}{16\eta t}-\frac{c_1\kappa_\eta}{\iota^5}\epsilon^2(\eta t+\iota),\label{eq:imp-or-loc-3}
    \end{align}
where we take an appropriate $c_1$ depending on $c$. By a union bound on \eqref{eq:imp-or-loc-3} for all $\tau\leq t$, we simply take maximum over all $\nm{\bm y_{t_0+\tau}-\bm y_{t_0}}^2$ to finish the proof.
\end{proof}

According to the result above, when the iterates move a long distance over a finite period (when $\max_{\tau\leq t}\nm{\bm y_{t_0+t}-\bm y_{t_0}}$ is large), the objective must receive a sufficient descent. On the contrary, if Algorithm \ref{alg:ALG} fails to significantly improve the objective, we conclude that $\max_{\tau\leq t}\nm{\bm y_{t_0+t}-\bm y_{t_0}}$ must be small and the iterates get stucked. This depicts an improve-or-localize behavior of Algorithm \ref{alg:ALG}. 

\subsection{Escaping saddle points}

Now, we consider an arbitrary $t_0$ such that $\bm x_{t_0}$ is an $\epsilon$-strict saddle point (see Definition \ref{defi:SOSP}), and denote $\bm H=\nabla^2f(\bm x_{t_0})$ for simplicity. Let $\bm v$ be the unit eigenvector corresponding to the eigenvalue $-\gamma:=\lambda_{\min}(\bm H)$. Recall that $L$-smoothness of $f$ gives rise to a double-sided bound of the spectrum of $\bm H$, i.e. any eigenvalue $\lambda(\bm H)\in[-L,L]$. Hence, when $\bm x_{t
_0}$ is an $\epsilon$-strict saddle point, $\bm H$ satisfies $\nm{\bm H}=|\lambda_{\max}(\bm H)|\leq L$ and $\lambda_{\min}(\bm H)\in[-L,-\sqrt{\rho\epsilon})$. 

We now define the concept of coupling sequences: the iterates generated by a pair of running instances of \ALG{}, with identical history information and symmetric randomness.

\begin{defi}[Coupling sequences]\label{defi:coup}

Let $\bm x_{t_0}$ be an $\epsilon$-strict saddle point, and denote $\bm H=\nabla^2f(\bm x_{t_0})$. Run two instances $A$, $A'$ of Algorithm \ref{alg:ALG}. Using the prime symbol ($\;'\;$) to distinguish the quantities generated by $A'$ from those in $A$, we suppose the two instances satisfy

(i) the history information prior to $t_0$ in $A$ and $A'$ is identical, i.e.
$$\bm x'_{t_0}=\bm x_{t_0},\quad\bm e'_{t_0}=\bm e_{t_0},\quad\bm e'_{t_0-1}=\bm e_{t_0-1},\quad\bm g'_{t_0-1}=\bm g_{t_0-1};$$

(ii) $A$ and $A'$ run with symmetric randomness after $t_0$, in that for each client $i$ and iteration $t$,
$$
\mathrm{FCC}'_p=\mathrm{FCC}_p,\quad\mathcal C'=\mathcal C,\quad\tilde{\nabla}_pf'_{i}=\tilde{\nabla}_pf_{i},\quad\bm\xi'_{t_0+t}=(\bm I-2\bm v\bm v^\top)\bm\xi_{t_0+t},
$$
where $\bm v$ is the unit eigenvector corresponding to $\lambda_{\min}(\bm H)$. Now, we say that $\{\bm x'_{t_0+t}\},\{\bm x_{t_0+t}\}$ are coupling sequences of iterates, and $\{\bm y'_{t_0+t}\},\{\bm y_{t_0+t}\}$ are coupling sequences of corrected iterates. Moreover, we use the hat symbol ($\;\hat{}\;$) to denote the difference between a pair of quantities generated by $A$ and $A'$, for example $\hbx_{t_0+t}:=\bm x'_{t_0+t}-\bm x_{t_0+t}$. 
\end{defi}

In our defined symmetry, $\bm\xi'_{t_0+t}$ reverts the component of $\bm\xi_{t_0+t}$ along the direction of $\bm v$, and keep other components intact. The symmetry of $\mathcal N(\bm 0,\frac{r^2}{d}\bm I)$ guarantees that their distributions are still identical. Combined with all the other symmetries in Defition \ref{defi:coup}, we conclude that the distributions of the coupling sequences are identical.

The difference between the coupling sequences of corrected iterates, $\hby_{t_0+t}$, admits a useful decomposition.

\begin{prop}[Proposition B.12, \citet{avdiukhin2021escaping}]\label{prop:coup-seq-decomp}
For any $t$, it holds that
$$\hby_{t_0+t}=-(\bm\Delta_t+(\bm E_t+\eta\hat{\bm e}_{t_0+t})+\bm Z_t+\bm\Xi_t),$$ where
\begin{align}
    \bm\Delta_t&:=\eta\sum_{\tau=0}^{t-1}(\bm I-\eta\bm H)^{t-\tau-1}\bm\delta_\tau\hbx_{t_0+\tau},\quad\bm\delta_\tau:=\int_0^1\nabla^2f(\alpha\bm x'_{t_0+\tau}+(1-\alpha)\bm x_{t_0+\tau})d\alpha-\bm H,\nonumber\\
    \bm E_t&:=\eta\sum_{\tau=0}^{t-1}(\bm I-\eta\bm H)^{t-\tau-1}(\hat{\bm e}_{t_0+\tau}-\hat{\bm e}_{t_0+\tau+1}),\nonumber\\
    \bm Z_t&:=\eta\sum_{\tau=0}^{t-1}(\bm I-\eta\bm H)^{t-\tau-1}\hat{\bm\zeta}_{t_0+\tau},\nonumber\\
    \bm\Xi_t&:=\eta\sum_{\tau=0}^{t-1}(\bm I-\eta\bm H)^{t-\tau-1}\hat{\bm\xi}_{t_0+\tau}.\nonumber
\end{align}
\end{prop}

According to Proposition \ref{prop:coup-seq-decomp}, $\hby_{t_0+t}$ decomposes into a sum of four terms, each showing the effect of one type of quantity that acumulates with time. Actually one can observe that $\bm\Delta_t$ reflects a cumulative dynamics of $\{\bm x_{t_0+t}\}$ and $\{\bm x'_{t_0+t}\}$, $\bm E_t+\eta\hat{\bm e}_{t_0+t}$ a cumulative compression error, $\bm Z_t$ a cumulative stochastic gradient noise, and $\bm\Xi_t$ a cumulative artificial perturbation.

In order to bound $\hby_{t_0+t}$, it is then a natural choice to bound each of the components respectively. 

\begin{lem}[Cumulative error bound]\label{lem:cumu-err-bound}
    Suppose that Assumption \ref{ass:f-max}, \ref{ass:f-L}, \ref{ass:nSG} hold. There exists a constant $c$ such that with probability at least $1-12te^{-\iota}$,
    \begin{equation*}
        \nm{\bm E_t+\eta\hat{\bm e}_{t_0+t}}\leq \frac{cL\eta}{\gamma}\sqrt{\Phi+\frac{\chi^2\iota}{\mu^2np}}(1+\eta\gamma)^t.
    \end{equation*}
\end{lem}
\begin{proof}
Expand the definition of $\bm E_t$,
\begin{align}
    \bm E_t&=\eta\left(\sum_{\tau=0}^{t-1}(\bm I-\eta\bm H)^{t-1-\tau}\hat{\bm e}_{t_0+\tau}-\sum_{\tau=1}^{t}(\bm I-\eta\bm H)^{t-\tau}\hat{\bm e}_{t_0+\tau}\right)\nonumber\\
    &=\eta\left(\sum_{\tau=1}^t(\bm I-\eta\bm H)^{t-1-\tau}(\bm I-\eta\bm H-\bm I)\hat{\bm e}_{t_0+\tau}+(\bm I-\eta\bm H)^{t-1}\hat{\bm e}_{t_0}-\hat{\bm e}_{t_0+t}\right)\nonumber\\
    &=-\eta\hat{\bm e}_{t_0+t}-\eta^2\bm H\sum_{\tau=1}^t(\bm I-\eta\bm H)^{t-1-\tau}\hat{\bm e}_{t_0+\tau}+\eta(\bm I-\eta\bm H)^{t-1}\hat{\bm e}_{t_0}\nonumber\\
    &=-\eta\hat{\bm e}_{t_0+t}-\eta^2\bm H\sum_{\tau=1}^t(\bm I-\eta\bm H)^{t-1-\tau}\hat{\bm e}_{t_0+\tau},\label{eq:c-err-bound-1}
\end{align}
where \eqref{eq:c-err-bound-1} is due to the initialization of coupling sequence, i.e. $\bm e_{t_0}=\bm e'_{t_0}$.
Hence
\begin{align}
    \nm{\bm E_t+\eta\hat{\bm e}_{t_0+t}}&\leq\nm{\eta^2\bm H\sum_{\tau=1}^t(\bm I-\eta\bm H)^{t-1-\tau}\hat{\bm e}_{t_0+\tau}}\nonumber\\
    &\leq\eta^2L\sum_{\tau=1}^t(1+\eta\gamma)^{t-1-\tau}\nm{\hat{\bm e}_{t_0+\tau}}\label{eq:c-err-bound-2}\\
    &\leq\eta^2L\sum_{\tau=1}^t(1+\eta\gamma)^{t-1-\tau}\cdot\max\{\nm{\bm e_{t_0+\tau}}+\nm{\bm e'_{t_0+\tau}}\}\nonumber\\
    &\leq c\sqrt{\Phi+\frac{\chi^2\iota}{\mu^2np}}\;\eta^2L\frac{2(1+\eta\gamma)^t}{\eta\gamma}\quad\text{ with prob. }1-12te^{-\iota}\label{eq:c-err-bound-3}\\
    &\leq\frac{2c\eta L}{\gamma}\sqrt{\Phi+\frac{\chi^2\iota}{\mu^2np}}(1+\eta\gamma)^t\quad\text{ with prob. }1-12te^{-\iota}.\nonumber
\end{align}
where we apply the spectral bound $\nm{\bm H}\leq L$ to \eqref{eq:c-err-bound-2} and apply Lemma \ref{lem:err-bound} to \eqref{eq:c-err-bound-3} respectively. Finally, $2c$ is picked as the constant.
\end{proof}
\begin{lem}[Artificial noise dynamics]\label{lem:art-noise}
    For any $t$, there exists a constant $c$ such that
    $$
    \nm{\bm\Xi_t}\leq\frac{c\sqrt{\iota\eta}r}{\sqrt{2np\gamma d}}(1+\eta\gamma)^t
    $$
    with probability at least $1-2e^{-\iota}$. Moreover, for $t\geq\frac{2}{\eta\gamma}$,
    $$
    \nm{\bm\Xi_t}\geq\frac{\sqrt{\eta}r}{3\sqrt{6np\gamma d}}(1+\eta\gamma)^t
    $$
    with probability at least $\frac{2}{3}$.
\end{lem}
\begin{proof}
This is a direct extension of Lemma 30, \citet{jin2019nonconvex}.
\end{proof}

\begin{lem}[Coupling sequence dynamics]\label{lem:coup-seq-dyn}
    Suppose that Assumption \ref{ass:f-max}, \ref{ass:f-L}, \ref{ass:f-rho}, \ref{ass:nSG} hold, and
    $$\max\{\nm{\bm y_{t_0+t}-\bm y_{t_0}},\nm{\bm y'_{t_0+t}-\bm y_{t_0}}\}\leq\mathcal R\quad\forall t\leq\mathcal I.$$
    Then, for any $t\geq\frac{2}{\eta\gamma}$, with probability at least $\frac{2}{3}-2t(6t+2d+1)e^{-\iota}$, we have
    \begin{equation*}
        \nm{\hby_{t_0+t}}\geq\frac{\sqrt{\eta}r}{6\sqrt{6np\gamma d}}(1+\eta\gamma)^t.
    \end{equation*}
\end{lem}
\begin{proof}
We will use induction to prove for all $t\geq0$ that $$\nm{\bm\Delta_t+(\bm E_t+\eta\hat{\bm e}_{t_0+t})+\bm Z_t}\leq\frac{\sqrt{\eta}r}{6\sqrt{6np\gamma d}}(1+\eta\gamma)^t$$
with probability at least $1-2t(6t+2d+1)e^{-\iota}$. With this at hand, we can then invoke Proposition \ref{prop:coup-seq-decomp} and Lemma \ref{lem:art-noise} to establish the desired lower bound for $\nm{\hby_{t_0+t}}$. 

The claim holds trivially at $t=0$. Now, suppose it holds as of $t-1$.

\textbf{Step 1: Bounding $\nm{\bm\Delta_t}$.} Consider any $\tau\leq t-1$. Under the assumption 
$$\max\{\nm{\bm y_{t_0+\tau}-\bm y_{t_0}},\nm{\bm y'_{t_0+\tau}-\bm y'_{t_0}}\}\leq\mathcal{R},$$
we have
\begin{align}
&\max\{\nm{\bm x_{t_0+\tau}-\bm x_{t_0}},\nm{\bm x'_{t_0+\tau}-\bm x_{t_0}}\}\nonumber\\
& \leq\max\{\nm{\bm y_{t_0+\tau}-\bm y_{t_0}},\nm{\bm y'_{t_0+\tau}-\bm y'_{t_0}}\}+\eta\max\{\nm{\bm e_{t_0+\tau}}+\nm{\bm e_{t_0}},\nm{\bm e'_{t_0+\tau}}+\nm{\bm e'_{t_0}}\}\nonumber\\
& \leq\mathcal{R}+2c\eta\sqrt{\Phi+\frac{\chi^2\iota}{\mu^2np}}\label{eq:coup-seq-dyn-1}\\
& \leq\mathcal{R}+\frac{2c\kappa_\eta\epsilon}{\iota^5L}\leq2\mathcal{R},\label{eq:coup-seq-dyn-2}
\end{align}
where Lemma \ref{lem:err-bound} yields \eqref{eq:coup-seq-dyn-1}, and \eqref{eq:coup-seq-dyn-2} holds by setting $\kappa_\eta\leq\frac{\kappa_\mathcal{R}}{2c}$.
Combined with Assumption \ref{ass:f-rho}, \eqref{eq:coup-seq-dyn-2} implies $\nm{\bm\delta_\tau}\leq 2\rho\mathcal{R}$. Now, by the inductive hypothesis and Lemma \ref{lem:art-noise},
\begin{align}
    \nm{\hbx_{t_0+\tau}}&\leq\nm{\hby_{t_0+\tau}}+\eta\nm{\hat{\bm e}_{t_0+\tau}}\leq\nm{\bm\Delta_\tau+(\bm E_\tau+\hat{\bm e}_{t_0+\tau})+\bm Z_\tau}+\nm{\bm\Xi_\tau}+\eta\nm{\hat{\bm e}_{t_0+\tau}}\nonumber\\
    &\leq\frac{\sqrt{\eta}r}{6\sqrt{6np\gamma d}}(1+\eta\gamma)^\tau+\frac{c_1\sqrt{\iota\eta}r}{\sqrt{2np\gamma d}}(1+\eta\gamma)^\tau+2c\eta\sqrt{\Phi+\frac{\chi^2\iota}{\mu^2np}}\label{eq:coup-seq-dyn-3}\\
    &\leq2c\sqrt{3\iota}\frac{\sqrt{\eta}r}{6\sqrt{6np\gamma d}}(1+\eta\gamma)^\tau,\label{eq:coup-seq-dyn-4}
\end{align}
where \eqref{eq:coup-seq-dyn-3} is again an application of Lemma \ref{lem:err-bound}, and \eqref{eq:coup-seq-dyn-4} holds if we set
$\frac{r^2}{np}\geq\frac{24c^2Ld}{c_1^2\iota}\left(\Phi+\frac{\chi^2\iota}{\mu^2np}\right)\eta$, which is implied by $\kappa_r\geq\frac{2c\sqrt{6\kappa_\eta}}{c_1}$. Then
\begin{align}
    \nm{\bm\Delta_t}&=\eta\nm{\sum_{\tau=0}^{t-1}(\bm I-\eta\bm H)^{t-1-\tau}\bm\delta_\tau\hbx_{t_0+\tau}}\nonumber\\
    &\leq\eta\sum_{\tau=0}^{t-1}(1+\eta\gamma)^{t-1-\tau}\cdot2\rho\mathcal{R}\cdot2c\sqrt{3\iota}\frac{\sqrt{\eta}r}{6\sqrt{6np\gamma d}}(1+\eta\gamma)^{\tau}\nonumber\\
    &\leq4c\sqrt{3\iota}\eta\mathcal{I}\rho\mathcal{R}\frac{\sqrt{\eta}r}{6\sqrt{6np\gamma d}}(1+\eta\gamma)^t\leq4\sqrt{3}c\kappa_\mathcal{R}\frac{\sqrt{\eta}r}{6\sqrt{6np\gamma d}}(1+\eta\gamma)^t\nonumber\\
    &\leq\frac{\sqrt{\eta}r}{18\sqrt{6np\gamma d}}(1+\eta\gamma)^t,\label{eq:coup-seq-dyn-5}
\end{align}
where we set $4\sqrt{3}c\kappa_\mathcal{R}\leq\frac{1}{3}$ for \eqref{eq:coup-seq-dyn-5}.

\textbf{Step 2: Bounding $\nm{\bm E_t+\eta\hat{\bm e}_{t_0+t}}$.} By Lemma \ref{lem:cumu-err-bound}, with probability at least $1-12te^{-\iota}$,
\begin{equation*}
    \nm{\bm E_t+\eta\hat{\bm e}_{t_0+t}}\leq\frac{cL\eta}{\gamma}\sqrt{\Phi+\frac{\chi^2\iota}{\mu^2np}}(1+\eta\gamma)^t\leq\frac{\sqrt{\eta}r}{18\sqrt{6np\gamma d}}(1+\eta\gamma)^t,
\end{equation*}
where the last inequality holds if we set $\frac{r^2}{np}\geq\frac{1944c^2L^2d}{\sqrt{\rho\epsilon}}\left(\Phi+\frac{\chi^2\iota}{\mu^2np}\right)\eta$, which is implied by
$\kappa_r\geq18c\sqrt{6\kappa_\eta}$.

\textbf{Step 3: Bounding $\nm{\bm Z_t}$.}
By Lemma 31, \citet{jin2019nonconvex}, with probability at least $1-4de^{-\iota}$, there exists a constant $c_2$ such that
\begin{equation*}
    \nm{\bm Z_t}\leq\frac{c_2\sigma\sqrt{\eta\iota\log d}}{\sqrt{2np\gamma}}(1+\eta\gamma)^t\leq\frac{\sqrt{\eta}r}{18\sqrt{6np\gamma d}}(1+\eta\gamma)^t,
\end{equation*}
where the last inequality holds by setting $\kappa_r\geq 18\sqrt{3}c_2$. 

\textbf{Step 4: Completing the induction.} By union bound, we have
\begin{equation*}
    \nm{\bm\Delta_t+(\bm E_t+\eta\hat{\bm e}_{t_0+t})+\bm Z_t}\leq\nm{\bm\Delta_t}+\nm{\bm E_t+\eta\hat{\bm e}_{t_0+t}}+\nm{\bm Z_t}\leq\frac{\sqrt{\eta}r}{6\sqrt{6np\gamma d}}(1+\eta\gamma)^t
\end{equation*}
with probability at least
\begin{equation*}
    1-2(t-1)(6(t-1)+2d+1)e^{-\iota}-2e^{-\iota}-12te^{-\iota}-4de^{-\iota}\leq1-2t(6t+2d+1)e^{-\iota},
\end{equation*}
which completes the induction. 
\end{proof}

From Lemma \ref{lem:coup-seq-dyn}, we observe that the difference between the coupling sequences has an exponential growth with time $t$, under the assumption that the iterates get stuck around the saddle points. Intuitively, after a sufficiently long period, it is contradictory to grow exponentailly and remain stuck at the same time. We now validate this intuition and show that the iterates generated by \ALG{} is able to escape the saddle points.

\begin{coro}[Escaping saddle points]\label{coro:escape-saddle}
    Suppose that Assumption \ref{ass:f-max}, \ref{ass:f-L}, \ref{ass:f-rho}, \ref{ass:nSG} hold. Then with probability at least $\frac{1}{3}-\mathcal I(6\mathcal I+2d+1)e^{-\iota}$, $$\max_{t\leq\mathcal I}\nm{\bm y_{t_0+t}-\bm y_{t_0}}\geq\mathcal R.$$
\end{coro}
\begin{proof}
We run two instances of Algorithm \ref{alg:ALG} according to Definition \ref{defi:coup} to obtain the coupling sequences $\{\bm y_{t_0+t}\},\{\bm y'_{t_0+t}\}$. Due to the identical distributions of $\{\bm y_{t_0+t}\}$ and $\{\bm y'_{t_0+t}\}$, it suffices to prove that the following event $\mathcal E$ holds with probability at least $\frac{2}{3}-2\mathcal I(6\mathcal I+2d+1)e^{-\iota}$:
$$
\max\{\nm{\bm y_{t_0+t}-\bm y_{t_0}},\nm{\bm y'_{t_0+t}-\bm y_{t_0}}\}\geq\mathcal R\quad\forall t\leq\mathcal I.
$$
Assume that $\mathcal E$ does not hold. By Lemma \ref{lem:coup-seq-dyn}, with probability at least $\frac{2}{3}-2\mathcal{I}(6\mathcal{I}+2d+1)e^{-\iota}$,
\begin{equation*}\label{eq:escape-saddle-1}
    \nm{\hby_{t_0+\mathcal{I}}}\geq\frac{\sqrt{\eta}r}{6\sqrt{6np\gamma d}}(1+\eta\gamma)^\mathcal{I}\geq2\mathcal R,
\end{equation*}
where we set
    $\mathcal{I}\geq\frac{\log\frac{12\mathcal{R}\sqrt{6npLd}}{\sqrt{\eta}r}}{\log(1+\eta\gamma)}$,
which is satisfied when ${\iota}\geq\log\frac{864npLd\mathcal{R}^2}{\eta r^2}$, meaning that $\iota$ can take $\tilde{O}(1)$ with respect to all the parameters. Then
$$\max\{\nm{\bm y_{t_0+\mathcal{I}}-\bm y_{t_0}},\nm{\bm y'_{t_0+\mathcal{I}}-\bm y_{t_0}}\}\geq\frac{1}{2}\nm{\hby_{t_0+\mathcal I}}\geq\mathcal{R},$$
which contradicts the assumption.
\end{proof}

\subsection{Convergence}
Combining Corollary \ref{coro:escape-saddle} with the improve-or-localize behavior of \ALG{} (Lemma \ref{lem:imp-or-loc}), we conclude that the objective receives sufficient descent.

\begin{lem}[Descent from saddles]\label{lem:saddle-desc}
Suppose that Assumption \ref{ass:f-max}, \ref{ass:f-L}, \ref{ass:f-rho}, \ref{ass:nSG} hold. Then with probability at least $1-7\mathcal Ie^{-\iota}$,
\begin{equation*}
    f(\bm y_{t_0+\mathcal{I}})-f(\bm y_{t_0})\leq\frac{1}{4}\mathcal{F}.
\end{equation*}
Moreover, with probability at least $\frac{1}{3}-2\mathcal{I}(3\mathcal{I}+d+4)e^{-\iota}$,
\begin{equation*}
    f(\bm y_{t_0+\mathcal{I}})-f(\bm y_{t_0})\leq-\mathcal{F}.
\end{equation*}
\end{lem}
\begin{proof}
By Lemma \ref{lem:imp-or-loc}, with probability at least $1-7\mathcal Ie^{-\iota}$,
\begin{align}
    f(\bm y_{t_0+\mathcal I})-f(\bm y_{t_0})&\leq \frac{c\kappa_\eta}{\iota^5}\epsilon^2(\eta\mathcal{I}+\iota)-\frac{1}{16\eta\mathcal{I}}\cdot\max_{t\leq\mathcal I}\nm{\bm y_{t_0+t}-\bm y_{t_0}}^2\nonumber\\
    &\leq\frac{2c\kappa_\eta}{\iota^4}\sqrt{\frac{\epsilon^3}{\rho}}-\frac{1}{16\eta\mathcal{I}}\cdot\max_{t\leq\mathcal I}\nm{\bm y_{t_0+t}-\bm y_{t_0}}^2\nonumber\\
    &\leq\frac{1}{4}\mathcal{F}-\frac{1}{16\eta\mathcal{I}}\cdot\max_{t\leq\mathcal I}\nm{\bm y_{t_0+t}-\bm y_{t_0}}^2,\label{eq:desc-from-saddle-1}
\end{align}
by setting $\kappa_\eta\leq\frac{\kappa_\mathcal{F}}{8c}$ in \eqref{eq:desc-from-saddle-1}.
The first claim is now immediate. To prove the second claim, invoking Corollary \ref{coro:escape-saddle}, we have
\begin{align}
    \frac{1}{16\eta\mathcal{I}}\cdot\max_{t\leq\mathcal I}\nm{\bm y_{t_0+t}-\bm y_{t_0}}^2\geq\frac{\mathcal{R}^2}{16\eta\mathcal{I}}=\frac{\kappa_\mathcal{R}^2}{16\iota^4}\sqrt{\frac{\epsilon^3}{\rho}}\geq\frac{5}{4}\mathcal{F}
\end{align}
with probability at least $\frac{1}{3}-\mathcal I(6\mathcal I+2d+1)e^{-\iota}$, where we set $\kappa_\mathcal{F}\leq\frac{\kappa_\mathcal{R}^2}{20}$. Taking this back to \eqref{eq:desc-from-saddle-1} implies that $f(\bm y_{t_0+\mathcal I})-f(\bm y_{t_0})\leq-\mathcal{F}$ with probability at least $\frac{1}{3}-2\mathcal I(3\mathcal I+d+4)e^{-\iota}$.
\end{proof}

We arrive at the final stage to show the convergence to $\epsilon$-SOSPs.

\begin{proof}[Proof of Theorem \ref{thm:SOSP}]
All the iterates can be classified into three types, namely (i) iterates that are not $\epsilon$-FOSPs, (ii) $\epsilon$-strict saddle points, and (iii) $\epsilon$-SOSPs. By Theorem \ref{thm:FOSP}, we have showed that at most 1/4 of the iterates are not $\epsilon$-FOSPs. Therefore, it suffices to show that at most 1/4 of the iterates are $\epsilon$-strict saddle points.

Similar to Theorem 16 of \citet{jin2019nonconvex}, we define the following stopping times $\{z_1,...,z_M\}$ by
\begin{align}
    z_1&=\inf\{\tau:\nm{\nabla f(\bm x_\tau)}\leq\epsilon\text{ and } \lambda_{\min}(f(\bm x_\tau))\leq-\sqrt{\rho\epsilon}\},\nonumber\\
    z_k&=\inf\{\tau>z_{i-1}+\mathcal{I}:\nm{\nabla f(\bm x_\tau)}\leq\epsilon\text{ and } \lambda_{\min}(f(\bm x_\tau))\leq-\sqrt{\rho\epsilon}\},\nonumber
\end{align}
with $M=\max\{k:z_k+\mathcal{I}\leq T\}$. Then we have
\begin{align}
    f(\bm y_T)-f(\bm y_0)&=\underbrace{\sum_{k=1}^M[f(\bm y_{z_k+\mathcal{I}})-f(\bm y_{z_k})]}_{T_1}\nonumber\\
    &\eqbr+\underbrace{[f(\bm y_T)-f(\bm y_{z_M})]+[f(\bm y_{z_1})-f(\bm y_0)]+\sum_{k=1}^{M-1}[f(\bm y_{z_{k+1}})-f(\bm y_{z_k+\mathcal{I}})]}_{T_2}.\nonumber
\end{align}
According to Lemma \ref{lem:saddle-desc} and a supermartingale concentration inequality, with probability at least $1-2\mathcal{I}(3\mathcal{I}+d+4)T^2e^{-\iota}$,
\begin{equation*}
    T_1\leq-\left(\frac{3}{4}M-c\sqrt{M\iota}\right)\mathcal{F}.
\end{equation*}
Applying union bound over all $t_0,t$ to Lemma \ref{lem:imp-or-loc}, with probability at least $1-7T^3e^{-\iota}$,
\begin{equation*}
    T_2\leq \frac{c\kappa_\eta}{\iota^5}\epsilon^2(\eta T+M\iota).
\end{equation*}
Suppose that more than $T/4$ iterates are $\epsilon$-strict saddle points, then $M\geq\frac{T}{4\mathcal{I}}$. Now with probability at least $1-T^2(6\mathcal I^2+2d\mathcal I+8\mathcal I+7T)e^{-\iota}\leq1-8T^2(\mathcal I^2+d\mathcal I+\mathcal I+T)e^{-\iota}$,
\begin{align}
    f(\bm y_T)-f(\bm y_0)&\leq-\left(\frac{3}{4}M-c\sqrt{M\iota}\right)\mathcal{F}+\frac{c\kappa_\eta}{\iota^5}\epsilon^2(\eta T+M\iota)\nonumber\\
    &\leq-\left(\frac{3}{4}M-c\sqrt{M\iota}\right)\mathcal{F}+\frac{c\kappa_\eta}{\iota^5}\epsilon^2(4\eta\mathcal{I}+\iota)M\nonumber\\
    &\leq-\left(\frac{3}{4}M-c\sqrt{M\iota}\right)\mathcal{F}+\frac{5c\kappa_\eta}{\iota^4}\sqrt{\frac{\epsilon^3}{\rho}}M\leq-\frac{1}{2}M\mathcal{F}\label{eq:SOSP-2}
    \\&\leq-\frac{T\mathcal{F}}{8\mathcal{I}},\label{eq:SOSP-1}
\end{align}
where we set $M\geq 64c^2\iota$ and $\kappa_\eta\leq\frac{\kappa_\mathcal{F}}{40c}$ for \eqref{eq:SOSP-2}. Clearly, by setting $\kappa_T>\frac{8}{\kappa_\mathcal{F}}$, we have
\begin{align*}
    T>\frac{8\mathcal{I}f_{\max}}{\mathcal{F}}=\frac{8\iota^5f_{\max}}{\kappa_\mathcal{F}\eta\epsilon^2}.
\end{align*}
Then \eqref{eq:SOSP-1} further gives $f(\bm y_T)-f(\bm y_0)<-f_{\max}$, which is a contradiction. This proves that at most 1/4 of the iterates are $\epsilon$-strict saddle points, hence establishes the theorem.
\end{proof}

\end{document}